\documentclass{article}
\usepackage[utf8]{inputenc}
\usepackage[round]{natbib}
\usepackage[margin=1in]{geometry}
\usepackage{amsmath}
\usepackage{amssymb}
\usepackage{amsthm}
\usepackage[colorlinks=true,urlcolor=blue,linkcolor=blue,citecolor=blue]{hyperref}%
\usepackage{mathtools}
\usepackage{graphicx}
\usepackage{subcaption}
\usepackage{array}
\usepackage{pifont}
\usepackage[ruled,vlined]{algorithm2e}
\newtheorem{thm}{Theorem}

\newcommand{\bv}{\boldsymbol{v}}
\newcommand{\bw}{\boldsymbol{w}}
\newcommand{\bx}{\boldsymbol{x}}
\newcommand{\by}{\boldsymbol{y}}

\newcommand{\btheta}{\boldsymbol{\theta}}

\newcommand{\bTheta}{\boldsymbol{\Theta}}

\newcolumntype{C}[1]{>{\centering\arraybackslash}m{#1}}
\title{Entropy Regularized Power \textit{k}-Means Clustering}
\author{Saptarshi Chakraborty\thanks{Joint first authors contributed equally to this work}$\,\,^{1}$,\, Debolina Paul$^{\ast 1}$,\, Swagatam Das$^2$,\, and Jason Xu\thanks{ Corresponding author: jason.q.xu@duke.edu}$\,\,^{3}$}
\date{}

\begin{document}

\maketitle

\begin{center}
$^1$Indian Statistical Institute, Kolkata, India \\
$^2$ Electronics and Communication Sciences Unit, Indian Statistical Institute, Kolkata, India\\
$^3$ Department of Statistical Science, Duke University, Durham, NC, USA. 

\normalsize
\end{center}

\begin{abstract}
Despite its well-known shortcomings, $k$-means remains one of the most widely used approaches to data clustering. Current research continues to tackle its flaws while attempting to preserve its simplicity. % through clever seeding and geometric arguments. 
Recently, the \textit{power $k$-means} algorithm was proposed to avoid trapping in local minima by annealing through a family of smoother surfaces. However,  the approach lacks theoretical justification and fails in high dimensions when many features are irrelevant. This paper addresses these issues by introducing \textit{entropy regularization} to learn feature relevance while annealing. We prove consistency of the proposed approach and derive a scalable majorization-minimization algorithm that enjoys closed-form updates and convergence guarantees. In particular, our method retains the same computational complexity of $k$-means and power $k$-means, but yields significant improvements over both. Its merits are thoroughly assessed on a suite of real and synthetic data experiments.
\end{abstract}
\section{Introduction}%\vspace{-0.1cm}
Clustering is a fundamental task in unsupervised learning for partitioning data into groups based on some similarity measure. %such that points within one group share maximum similarity and points within different groups are as dissimilar as possible. \par
Perhaps the most popular approach  is $k$-means clustering \citep{macqueen1967some}: given a dataset $\mathcal{X}=\{\bx_1,\dots,\bx_n\} \subset\mathbb{R}^p$,  $\mathcal{X}$ is to be partitioned into $k$ mutually exclusive classes so that the variance within each cluster is minimized. The problem can be cast as minimization of the objective
%Let $\athcal{X}=\{\mathbf{x}_1,\dots,\mathbf{x}_n\} \subset \mathbb{R}^p$ be a set of $n$ data-points in $\mathbb{R}^p$. $\mathcal{X}$ is to be partitioned into $k$ exhaustive and mutually exclusive partitions or clusters. The $k$-means problem can be formulated as the minimization of the following objective function:
\begin{equation}\label{obj1} P(\bTheta) = \sum_{i=1}^n \min_{1 \le j \le k} \|\bx_i-\btheta_j\|^2,
 \end{equation}
%\begin{equation}
%\label{obj1}
%P(\mathcal{U}, %\Theta)=\sum_{j=1}^{k} %\sum_{i=1}^{n}u_{ij}\|\bx_i-%\btheta_j\|^2,
%\end{equation}
%where $\mathcal{U}$ is an $n\times k$ cluster assignment matrix,
%\[u_{ij}=
%    \begin{dcases}
%        1 & \mathbf{x}_i \in %j\text{-th Cluster} \\
%        0 & \text{Otherwise} %
%    \end{dcases}
%\]
where 
$\bTheta=\{{\btheta}_1, {\btheta}_2, \dots , {\btheta}_k\}$ denotes the set of cluster centroids, and $\|\bx_i-\btheta_j\|^2$ is the usual squared Euclidean distance metric. %which, may be replaced by other dissimilarity measures as is done in many recent variants of the algorithm. 

Lloyd's algorithm \citep{lloyd1982least}, which iterates between assigning points to their nearest centroid and updating each centroid by averaging over its assigned points, is the most frequently used heuristic to solve the preceding minimization problem. 
Such heuristics, however, suffer from several well-documented drawbacks. Because the task is NP-hard \citep{aloise2009np}, Lloyd's algorithm and its variants seek to approximately solve the problem and are prone to stopping at poor local minima, especially as the number of clusters $k$ and dimension $p$ grow. Many new variants have since contributed to a vast literature on the topic, including  spectral clustering \citep{DBLP:conf/nips/NgJW01}, Bayesian \citep{lock2013bayesian} and non-parametric methods \citep{DBLP:conf/icml/KulisJ12}, subspace clustering \citep{DBLP:journals/spm/Vidal11},  sparse clustering \citep{witten2010framework}, and convex clustering \citep{chi2015splitting}; a more comprehensive overview can be found in \cite{jain2010data}. 

None of these methods have managed to supplant $k$-means clustering, which endures as the most widely used approach among practitioners due to its simplicity. Some work instead focuses on ``drop-in" improvements of Lloyd's algorithm. The most prevalent strategy is clever seeding: $k$-means++ \citep{arthur2007k,ostrovsky2012effectiveness} is one such effective wrapper method in theory and practice, and proper initialization methods remain an active area of research \citep{celebi2013comparative,bachem2016fast}.
Geometric arguments have also been employed to overcome sensitivity to initialization. \cite{zhang1999k} proposed to replace the minimum function by the harmonic mean function to yield a smoother objective function landscape but retain a similar algorithm, though the strategy fails in all but very low dimensions. \citet{xu2019power} generalized this idea by using a sequence of successively smoother objectives via power means instead of the harmonic mean function to obtain better approximating functions in each iteration.
The contribution of power $k$-means is algorithmic in nature---it effectively avoids local minima from an \textit{optimization} perspective, and succeeds for large $p$ when the data points are well-separated. However, it does not address the \textit{statistical} challenges in high-dimensional settings and performs as poorly as standard $k$-means in such settings. A meaningful similarity measure plays a key role in revealing clusters \citep{de2012minkowski,chakraborty2017k}, but pairwise Euclidean distances become decreasingly informative as the number of features grows due to the curse of dimensionality. %We find that the approach performs as poorly as standard $k$-means in such scenarios. A meaningful similarity measure plays a key role in revealing clusters \citep{de2012minkowski,chakraborty2017k}, but pairwise Euclidean distances become decreasingly informative as the number of features grows due to the curse of dimensionality.
\par

On the other hand, there is a rich literature on clustering in high dimensions, but standard approaches such as subspace clustering are not scalable due to the use of an affinity matrix pertaining to norm regularization  \citep{ji2014efficient,abcd}. %For such methods, serious computational burdens stem from large scale optimization problems induced by the need for creating an affinity matrix. 
For spectral clustering, even the creation of such a matrix quickly becomes intractable for modern, large-scale problems \citep{zhang2019neural}.
%Moreover, the requirement of spectral clustering on the affinity matrix, the size of which depends on the number of data points, add to the computational complexity as well \cite{zhang2019neural}. 
Toward learning  effective feature representations,
\cite{huang2005automated} proposed weighted $k$-means clustering ($WK$-means), %and \cite{jing2007entropy} used entropy regularization in the context of $W$-$k$-means clustering to achieve more efficient feature weighting. 
and sparse $k$-means \citep{witten2010framework} has become a benchmark feature selection algorithm, where  selection is achieved by imposing  $\ell_1$ and $\ell_2$ constraints on the feature weights. %\cite{jin2016influential} proposed the IF-HCT-PCA algorithm, which uses Higher Criticism (HC) thresholding and Principal Components Analysis (PCA) to reduce the dimensionality of the data. 
Further related developments can be found in the works of \cite{modha2003feature,li2006novel,huang2008weighting,de2012minkowski,jin2016influential}. These approaches typically lead to complex optimization problems in terms of transparency as well as computational efficiency---for instance, sparse $k$-means  requires solving constrained sub-problems via bisection to find the necessary dual parameters $\lambda^\ast$ in evaluating the proximal map of the $\ell_1$ term. As they fail to retain the simplicity of Lloyd's algorithm for $k$-means, they lose appeal to practitioners. Moreover, these works on feature weighing and selection do not benefit from recently algorithmic developments as mentioned above. 
\par 
In this article, we propose a scalable clustering algorithm for high dimensional settings that leverages recent insights for avoiding poor local minima, performs adaptive feature weighing, and preserves the low complexity and transparency of $k$-means. Called Entropy Weighted Power $k$-means (EWP), we extend the merits of power $k$-means to the high-dimensional case by introducing feature weights together with entropy incentive terms. Entropy regularization is not only effective both theoretically and empirically, but leads to an elegant algorithm with closed form solution updates. The idea is to minimize along a continuum of smooth surrogate functions that gradually approach the $k$-means objective, while the feature space also gradually adapts so that clustering is driven by informative features.  By transferring the task onto a sequence of better-behaved optimization landscapes, the algorithm fares better against the curse of dimensionality and against adverse initialization of the cluster centroids than existing methods. 
\begin{figure}[ht]
    \centering \hspace{-15pt}
    \begin{subfigure}[b]{0.34\textwidth} 
    \centering
        \includegraphics[height=\textwidth,width=\textwidth]{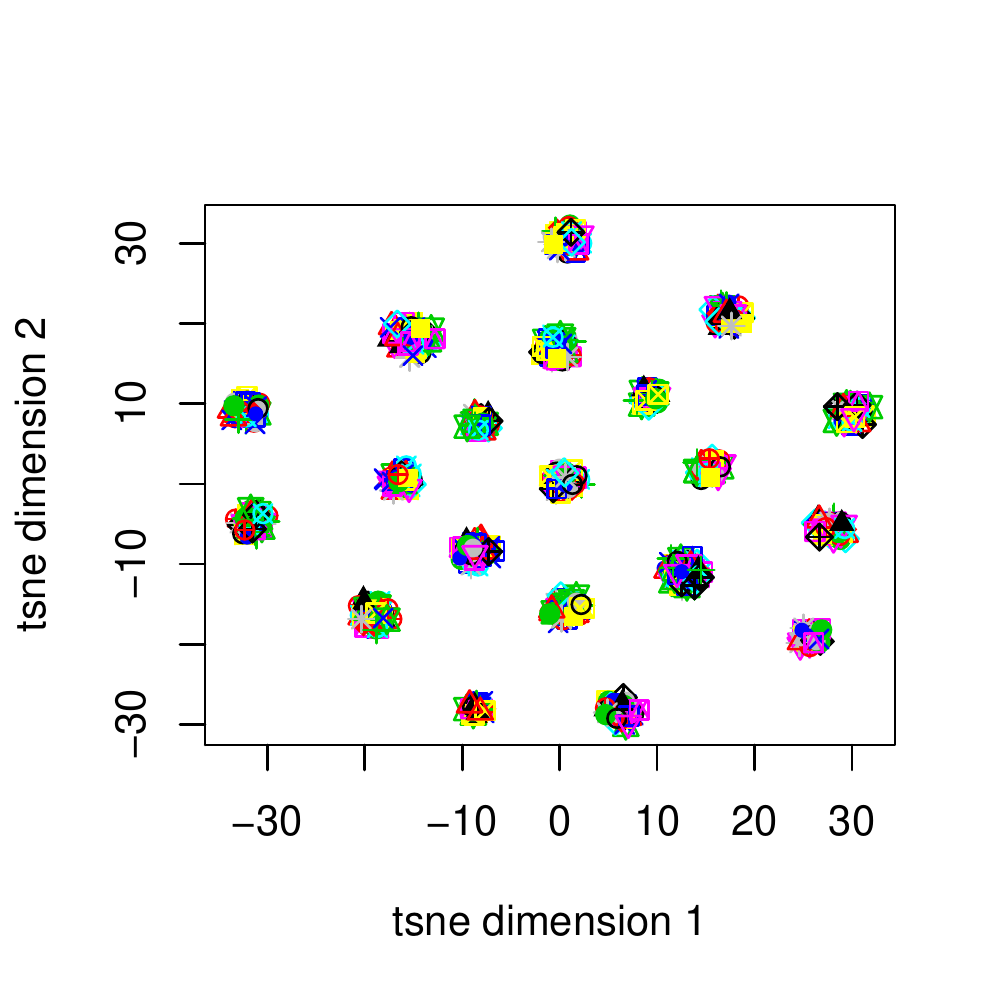}
        \caption{$k$-means}
        %\label{gt}
    \end{subfigure}
    ~
    \hspace{-17pt}
    \begin{subfigure}[b]{0.34\textwidth} 
    \centering
        \includegraphics[height=\textwidth,width=\textwidth]{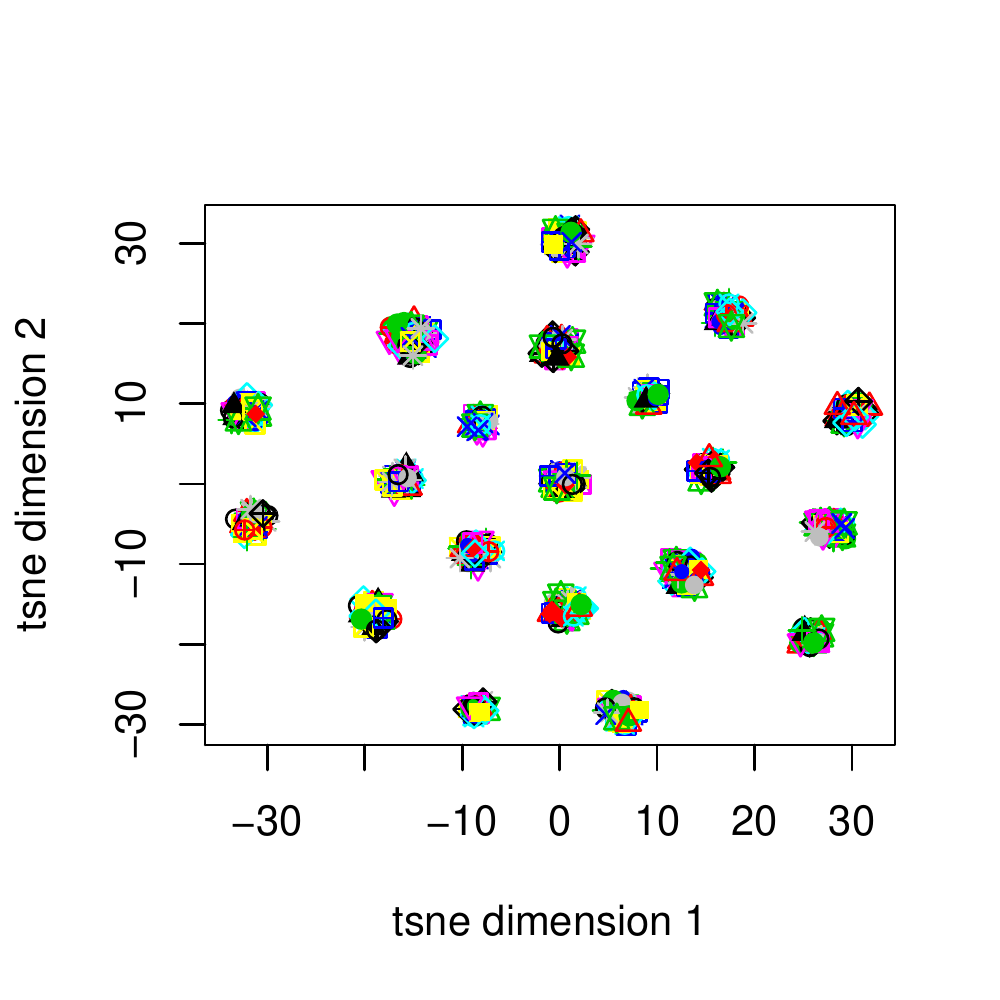}
        \caption{$WK$-means}
        %\label{gt}
    \end{subfigure}
    ~
    \hspace{-17pt}
     \begin{subfigure}[b]{0.34\textwidth} 
    \centering
        \includegraphics[height=\textwidth,width=\textwidth]{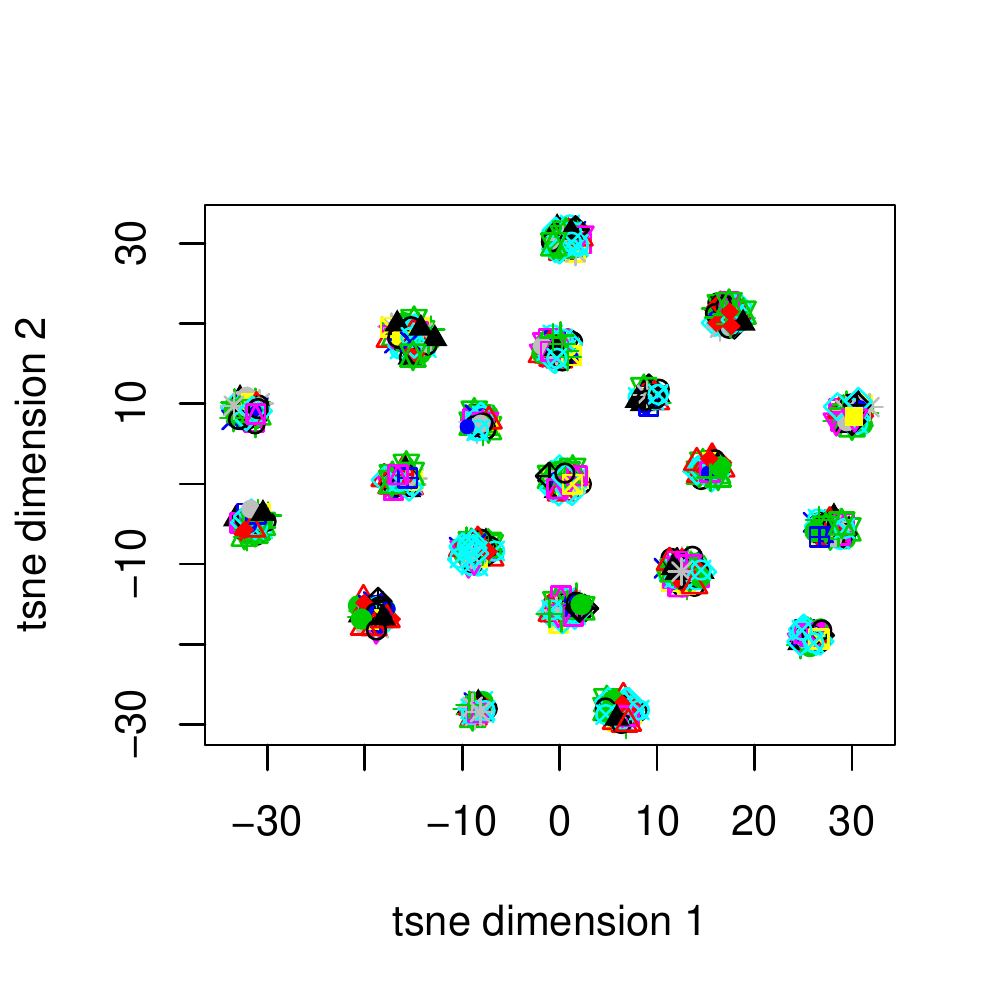}
        \caption{Power $k$-means}
        %\label{hist}
    \end{subfigure}
    ~
    \hspace{-17pt}
 \begin{subfigure}[b]{0.34\textwidth} 
    \centering
        \includegraphics[height=\textwidth,width=\textwidth]{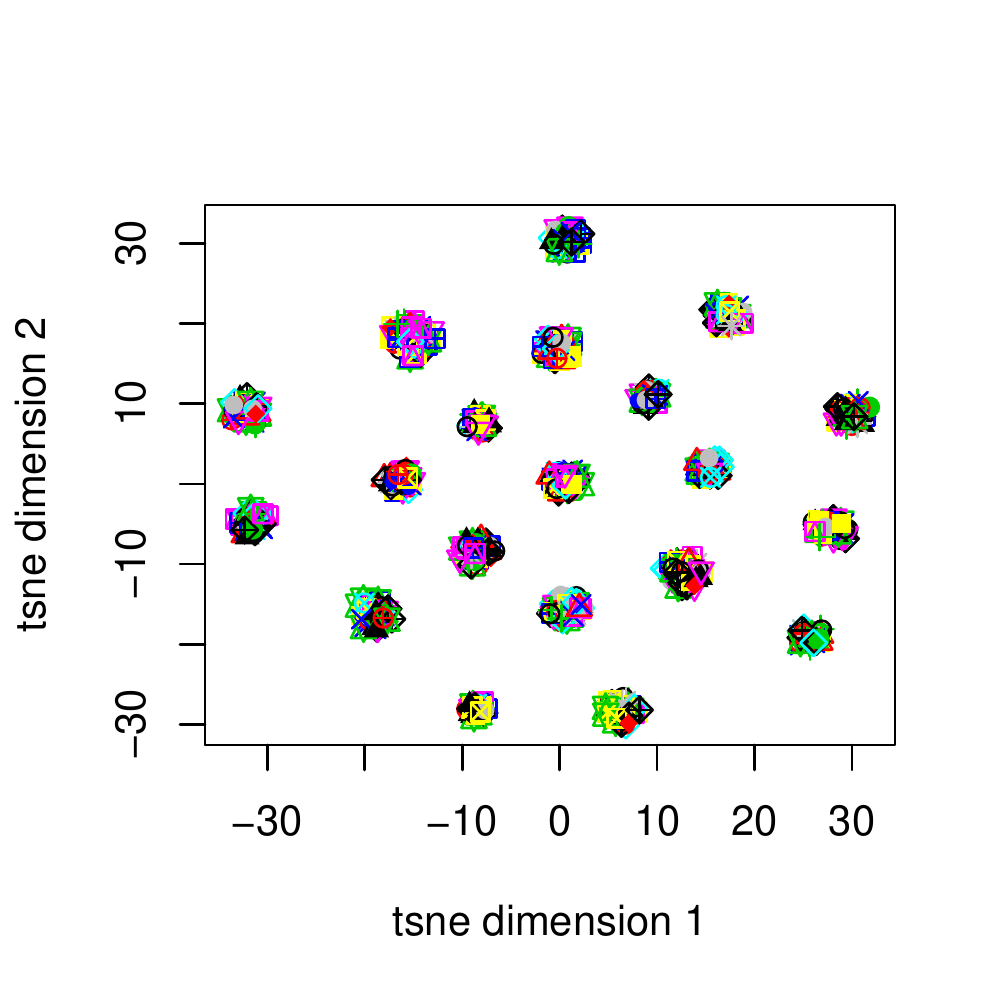}
        \caption{Sparse $k$-means}
       % \label{hist}
    \end{subfigure}
    ~
    \hspace{-17pt}
 \begin{subfigure}[b]{0.34\textwidth} 
    \centering
        \includegraphics[height=\textwidth,width=\textwidth]{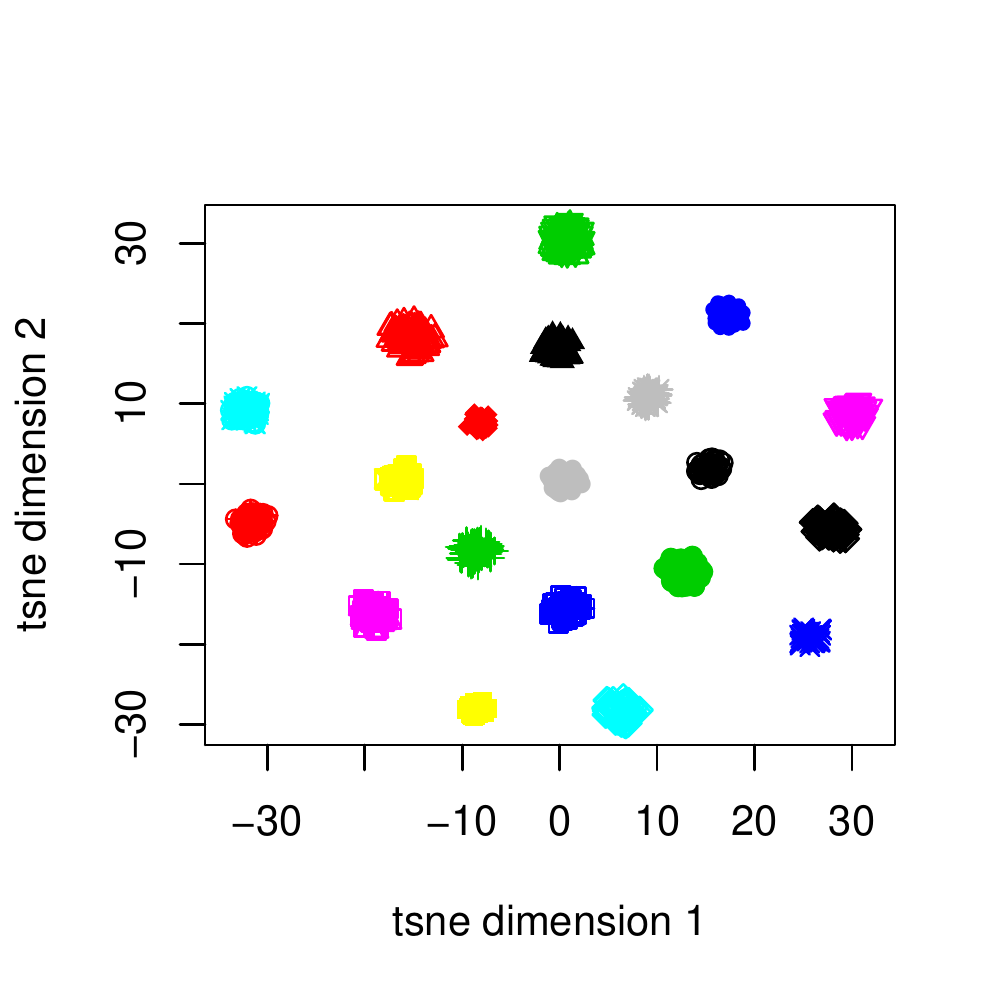}
        \caption{EWP}
        %\label{hist}
    \end{subfigure}
            \caption{Peer methods fail to cluster in $100$ dimensions with $5$ effective features on illustrative example, while the proposed method achieves perfect separation. Solutions are visualized using \texttt{t-SNE}.}
            \label{fig:eg1}
\end{figure}
The following summarizes our main contributions: %\vspace{-6pt}
\begin{itemize}
    \item We propose a clustering framework that automatically learns a weighted feature representation while simultaneously avoiding local minima through annealing.
    \item We develop a scalable Majorization-Minimization (MM) algorithm to minimize the proposed objective function.
    \item We establish descent and convergence properties of our method and prove the strong consistency of the global solution.
    \item Through an extensive empirical study on real and simulated data, we demonstrate the efficacy of our algorithm, finding that it outperforms comparable classical and state-of-the-art approaches.
\end{itemize} %\vspace{-6pt}
 %In particular, the effects of the increasing number of features and clusters are discussed in Sections \ref{simulation 1} and \ref{set2}  respectively. 
%We then focus on a case study of micro-array data in Section \ref{glioma}, and close with a discussion of our contributions.

The rest of the paper is organized as follows. After reviewing some necessary background, Section \ref{erpk} formulates the Entropy Weighted Power $k$-means (EWP) objective and provides high-level intuition. Next, an MM algorithm to solve the resulting optimization problem is derived in Section \ref{optimization}. Section \ref{theory} establishes the theoretical properties of the EWP clustering. Detailed experiments on both real and simulated datasets are presented in Section \ref{experiment}, followed by a discussion of our contributions in Section \ref{discussion}.

\paragraph{Majorization-minimization}
The principle of MM has become increasingly popular for large-scale optimization in statistical learning \citep{mairal2015incremental,lange2016mm}.
 %Indeed, Lloyd's algorithm can be interpreted as an EM algorithm for a Gaussian mixture model (GMM) with vanishing variances or as a variational EM approximation with isotropic GMMs \cite{lucke2017}, it is natural that the broader MM principle underpins our own method.
Rather than minimizing an objective of interest $f(\cdot)$ directly, an MM algorithm successively minimizes a sequence of simpler \textit{surrogate functions} $g(\btheta \mid \btheta_n)$ that  \textit{majorize} the original objective  $f(\btheta)$ at the current estimate 
$\btheta_m$. Majorization requires two conditions: tangency  $g(\btheta_m \mid \btheta_m) =  f(\btheta_m)$ at the current iterate, and domination $g(\btheta \mid \btheta_m)  \geq f(\btheta)$ for all $\btheta$. The iterates of the MM algorithm are defined by the rule 
\begin{equation}\label{eq:MMiter}
\btheta_{m+1} := \arg\min_{\btheta}\; g(\btheta \mid \btheta_m) \end{equation}
which immediately implies the descent property
\begin{eqnarray*}
f(\btheta_{m+1}) \, \leq \, g(\btheta_{m+1} \mid \btheta_{m}) 
\, \le \,  g(\btheta_{m} \mid \btheta_{m}) 
\, = \, f(\btheta_{m}). \label{eq:descent}
\end{eqnarray*}
That is, a decrease in $g$ results in a decrease in $f$.
Note that $g(\btheta_{m+1} \mid \btheta_{m} ) \le g(\btheta_{m} \mid \btheta_{m})$ does not require $\btheta_{m+1}$ to minimize $g$ exactly, so that any descent step in $g$ suffices. 
%Similarly, maximizing a function can be accomplished analogously via sequential minorization-maximization. 
The MM principle offers a general prescription for transferring a difficult optimization task onto a sequence of simpler problems \citep{LanHunYan2000}, and includes the well-known EM algorithm for maximum likelihood estimation under missing data as a special case \citep{BecYanLan1997}. 
\paragraph{Power k-means}
\cite{zhang1999k} attempt to reduce sensitivity to initialization in $k$-means by minimizing the criterion
\begin{eqnarray}
 \sum_{i=1}^n \Big(\frac{1}{k}  \sum_{j=1}^k \|\bx_i-\btheta_j\|^{-2} \Big)^{-1} := f_{-1}(\bTheta ). \label{KHM}
\end{eqnarray}
Known as $k$-harmonic means, the method replaces the $\min$ appearing in \eqref{obj1} by the harmonic average to yield a smoother optimization landscape, an effective approach in low dimensions. %Though effective, this proxy becomes less appropriate as the dimension increases, and no longer preserves the within-cluster variance interpretation of \eqref{obj1}. 
Recently, power $k$-means clustering extends this idea to work in higher dimensions where \eqref{KHM} is no longer a good proxy for \eqref{obj1}.
Instead of considering only the closest centroid or the harmonic average, the \textit{power mean} 
between each point and all $k$ centroids provides a family of successively smoother optimization landscapes. The power mean of a vector $\by$ is defined
$M_s(\by) = \left( \frac{1}{k}\sum_{i=1}^k y_i^s \right)^{1/s}. $
Within this class, $s>1$ corresponds to the usual $\ell_s$-norm of $\by$, $s=1$ to the arithmetic mean, and  $s=-1$ to the harmonic mean. 

Power means enjoy several nice properties that translate to algorithmic merits and are useful for establishing theoretical guarantees. They are monotonic, homogeneous, and differentiable with gradient 
\begin{eqnarray}
 \frac{\partial}{\partial y_j}   M_ s( \by) & =&  \Big(\frac{1}{k}\sum_{i=1}^k y_i^s\Big)^{\frac{1}{s}-1} \frac{1}{k}y_j^{s-1} ,\label{power_mean_grad}
\end{eqnarray}
and satisfy the limits
\begin{subequations}\label{eq:limit}
\begin{equation}
    \lim_{s \to -\infty}M_s(\by)=\min\{y_1,\ldots,y_k\}
\end{equation}
\begin{equation}
    \lim_{s \to \infty}M_s(\by)=\max\{y_1,\ldots,y_k\} .
\end{equation}
\end{subequations}
Further, the well-known power mean inequality
$M_s (\by) \le M_ t (\by)$ for any $s \le t$ holds \citep{steele2004cauchy}.
%As $s \rightarrow -\infty$, the power mean approaches the min function due to \eqref{eq:limit}, and thus, the objective agrees with the $k$-means criterion in this limit. 

The power $k$-means objective function for a given power $s$ is given by the formula
\begin{equation}\label{eq:limit2}
f_s(\Theta )=\sum_{i=1}^n M_s(\|\bx_i-\btheta_1\|^2,\ldots,\|\bx_i-\btheta_k\|^2).
\end{equation}
The algorithm then seeks to minimize $f_s$ iteratively while sending $s \rightarrow -\infty$. Doing so, the objective approaches $f_{-\infty}(\Theta)$ due to \eqref{eq:limit}, coinciding with the original $k$-means objective and retaining its interpretation as minimizing within-cluster variance. The intermediate surfaces provide better optimization landscapes that exhibit fewer poor local optima than \eqref{obj1}. Each minimization step is carried out via MM; see \cite{xu2019power} for details.

\section{Entropy Weighted Power \textit{k}-means}%\vspace{-0.1cm}
%We derive a clustering algorithm that addresses both statistical and algorithmic problems arising in high-dimensional clustering while retaining the simplicity of $k$-means. 
%We begin with a simple illustrative example before formulating our approach and deriving our algorithm. %We begin by reviewing some background.
\paragraph{A Motivating Example}
\label{motivation}
We begin by considering a synthetic dataset with $k=20$ clusters, $n=1000$ points, and $p=100$. Of the $100$ features, only $5$ are relevant for distinguishing clusters, while the others are sampled from a standard normal distribution (further details are described later in  Simulation 2 of Section \ref{set2}).  We compare standard $k$-means, $WK$-means, power $k$-means, and sparse $k$-means with our proposed method; sparse $k$-means is tuned using the gap statistic described in the original paper \citep{witten2010framework} as implemented in the \texttt{R} package, \texttt{sparcl}. Figure \ref{fig:eg1} displays the solutions in a $t$-distributed Stochastic Neighbourhood Embedding (\texttt{t-SNE}) \citep{maaten2008visualizing} for easy visualization in two dimensions. It is evident that our EWP algorithm, formulated below, yields perfect recovery while the peer algorithms fail to do so. This transparent example serves to illustrate the need for an approach that simultaneously avoids poor local solutions while accommodating high dimensionality.

\subsection{Problem Formulation}%\vspace{-0.1cm}
\label{erpk}
Let $\bx_1,\dots,\bx_n \in \mathbb{R}^p$ denote the $n$ data points, and $\Theta_{k \times p}=[\btheta_1,\dots,\btheta_k]^\top$ denote the matrix whose rows contain the cluster centroids. We introduce a feature relevance vector $\bw \in \mathbb{R}^p$ where $w_l$ contains the weight of the $l$-th feature, and require these weights to satisfy the constraints %(\ref{c1}) and (\ref{c2}). 
\begin{equation}
\tag{C}
\sum_{l=1}^p w_l=1; \qquad w_l  \geq 0 \text{ for all } l=1,\dots,p .
\label{c1}
\end{equation}
%\begin{equation}
%\tag{C2}
%\sum_{l=1}^p w_l=1
%\label{c2}
%\end{equation}
The EWP objective for a given $s$ is now given by
\begin{equation}
\label{obj}
f_s(\Theta,\bw) = \sum_{i=1}^n M_s(\|\bx-\btheta_1\|_{\bw}^2,\dots,\|\bx-\btheta_k\|_{\bw}^2) + \lambda \sum_{l=1}^p w_l \log w_l,
\end{equation}
where the \textit{weighted} norm $\|\by\|_{\bw}^2=\sum_{l=1}^p w_l y_l^2$ now appears as arguments to the power mean $M_s$.  The final term is the negative entropy of $\bw$ \citep{jing2007entropy}. This entropy incentive is minimized when $w_l=1/p$ for all $l=1,\dots,p$; in this case, equation (\ref{obj}) is equal to the power $k$-means objective, which in turn equals the $k$-means objective when $s\rightarrow -\infty$ (and coincides with KHM for $s=-1$). EWP thus generalizes these approaches, while newly allowing features to be adaptively weighed throughout the clustering algorithm. Moreover, we will see in Section \ref{optimization} that entropy incentives are an ideal choice of regularizer in that they lead to closed form updates for $\bw$ and $\btheta$ within an iterative algorithm.

\paragraph{Intuition and the curse of dimensionality}
\label{intuition}
 %This is desirable since the $k$-means objective is always interpretable in terms of within-cluster variance, etc, while any given power mean does not have an obvious useful interpretation. Once the harmonic mean does not resemble the min function, for instance, it does not have any use toward clustering in itself. 
Power $k$-means combats the curse of dimensionality by providing smoothed \textit{objective functions} that remain appropriate as dimension increases. Indeed, in practice the value of $s$ at convergence of power $k$-means becomes lower as the dimension increases, explaining its outperformance over $k$-harmonic means \citep{zhang1999k}--- $f_{-1}$ deteriorates as a reasonable approximation of $f_{-\infty}$.  
%In particular, power $k$-means combats the curse of dimensionality from the perspective of the \textit{objective function}.
However even if poor solutions are successfully avoided from the algorithmic perspective, the curse of dimensionality still affects the \textit{arguments} to the objective. Minimizing within-cluster variance becomes less meaningful as pairwise Euclidean distances become uninformative in high dimensions \citep{aggarwal2001surprising}.
% As mentioned earlier, even though the power mean averages $k$ arguments/inputs, each of these arguments is a Euclidean distance between two points in dimension $d$. When $d$ is large, it is well known that these distances become less and less informative. 
It is therefore desirable to reduce the effective dimension in which distances are computed. %One way to do so is, of course, to do feature selection in the classic sense--- that is, to select only a small number of nonzero features. However, if one were to have weights on all features, and only a small number have large weights while the others have trivial weight, the same effect is produced.

While the entropy incentive term does not zero out variables, it weighs the dimensions according to how useful they are in driving clustering. When the data live in a high-dimensional space yet only a small number of features are relevant towards clustering, the optimal solution to our objective \eqref{obj} assigns non-negligible weights to only those few relevant features, while benefiting from annealing through the weighted power mean surfaces. %The algorithm thus extends previous approaches by fighting the curse of dimensionality through the appropriate objective function as well as informative arguments into this function.
\subsection{Optimization}%\vspace{-0.1cm}
\label{optimization}
To optimize the EWP objective, we develop an MM algorithm \citep{lange2016mm} for sequentially minimizing  (\ref{obj}). As shown by \cite{xu2019power}, $M_s(\by)$ is concave if $s<1$; in particular, it lies below its tangent plane. This observation provides the following inequality: denoting $\by_m$ the estimate of a variable $\by$ at iteration $m$,
\begin{equation}
\label{mm1}
M_s(\by) \leq M_s(\by_m)+ \nabla_{\by} M_s(\by_m)^\top (\by-\by_m)
\end{equation}
Substituting $\|\bx_i-\btheta_j\|_{\bw}^2$ for $y_j$ and $\|\bx_i-\btheta_{mj}\|_{\bw_m}^2$ for $y_{mj}$ in equation (\ref{mm1}) and summing over all $i$, we obtain 
{\small
\begin{align*}
&f_s(\bTheta,\bw) \leq  f_s(\bTheta_m,\bw_m)-\sum_{i=1}^n \sum_{j=1}^k \phi_{ij}^{(m)} \|\bx_i-\btheta_{mj}\|_{\bw_m}^2\\
&-\lambda \sum_{l=1}^p (w_{m,l} \log w_{m,l}-w_l \log w_l) + \sum_{i=1}^n \sum_{j=1}^k \phi_{ij}^{(m)}\|\bx_i-\btheta_{j}\|_{\bw}^2.
\end{align*}
}%
Here the derivative expressions \eqref{power_mean_grad} provide the values of the constants $$\phi_{ij}^{(m)}=\frac{\frac{1}{k}\|\bx_i-\btheta_{m,j}\|_{\bw_m}^{2(s-1)}}{\bigg(\frac{1}{k}\sum_{j=1}^k\|\bx_i-\btheta_{m,j}\|_{\bw_m}^{2s}\bigg)^{(1-\frac{1}{s})}}.$$
The right-hand side of the inequality above serves as a \textit{surrogate function} majorizing $f_s(\bTheta, \bw)$ at the current estimate $\bTheta_m$. Minimizing this surrogate amounts to minimizing the expression
\begin{equation}\label{eq7}
 \sum_{i=1}^n \sum_{j=1}^k \phi_{ij}^{(m)} \|\bx_i-\btheta_{j}\|_{\bw}^2+\lambda\sum_{l=1}^p w_l \log w_l
\end{equation}
subject to the constraints (\ref{c1}). This problem admits closed form solutions: minimization over $\bTheta$ is straightforward, and the optimal solutions are given by
$$\btheta_{j}^* =\frac{\sum_{i=1}^n \phi_{ij} \bx_i}{\sum_{i=1}^n \phi_{ij}}.$$
To minimize equation (\ref{eq7}) in $\bw$, we consider the Lagrangian
$$\mathcal{L}= \sum_{i=1}^n \sum_{j=1}^k \phi_{ij} \|\bx_i-\btheta_{j}\|_{\bw}^2+\lambda\sum_{l=1}^p w_l \log w_l-\alpha (\sum_{l=1}^p w_l-1).$$
The optimality condition $\frac{\partial \mathcal{L}}{\partial w_l}=0$ implies  $\sum_{i=1}^n \sum_{j=1}^k \phi_{ij} (x_{il}-\theta_{jl})^2+\lambda(1+ \log w_l)-\alpha=0. $ This further implies that 

$$w_l^*  \propto \exp{\bigg\{-\frac{\sum_{i=1}^n \sum_{j=1}^k \phi_{ij}(x_{il}-\theta_{jl})^2}{\lambda}\bigg\}}.$$
Now enforcing the constraint $\sum_{l=1}^p w_l =1$, we get
$$w_l^* = \frac{\exp{\bigg\{-\frac{\sum_{i=1}^n \sum_{j=1}^k \phi_{ij}(x_{il}-\theta_{jl})^2}{\lambda}\bigg\}}}{\sum_{t=1}^p \exp{\bigg\{-\frac{\sum_{i=1}^n \sum_{j=1}^k \phi_{ij}(x_{it}-\theta_{jt})^2}{\lambda}\bigg\}}}.$$ 
Thus, the MM steps take a simple form and  amount to two alternating updates:
\begingroup
\allowdisplaybreaks
\begin{align}
    \btheta_{m+1,j} & =\frac{\sum_{i=1}^n \phi_{ij}^{(m)} \bx_i}{\sum_{i=1}^n \phi_{ij}^{(m)}}\\
    w_{m+1,l} & = \frac{\exp{\bigg\{-\frac{\sum_{i=1}^n \sum_{j=1}^k \phi_{ij}^{(m)}(x_{il}-\theta_{jl})^2}{\lambda}\bigg\}}}{\sum_{t=1}^p \exp{\bigg\{-\frac{\sum_{i=1}^n \sum_{j=1}^k \phi_{ij}^{(m)}(x_{it}-\theta_{jt})^2}{\lambda}\bigg\}}}.
\end{align}
\endgroup
The MM updates are similar to those in Lloyd's algorithm \citep{lloyd1982least} in the sense that each step alternates between updating $\phi_{ij}$'s and updating $\bTheta$ and $\bw$. These updates are summarised in Algorithm \ref{alg}; though there are three steps rather than two, the overall per-iteration complexity of this algorithm is the same as that of $k$-means (and power $k$-means) at $\mathcal{O}(nkp)$ \citep{lloyd1982least}. We require the tuning parameter $\lambda>0$ to be specified, typically chosen via cross-validation detailed in Section \ref{simulation}. It should be noted that the initial value $s_0$ and the constant $\eta$ do not require careful tuning: we fix them at $s_0=-1$ and $\eta=1.05$ across \textit{all} real and simulated settings considered in this paper.

%\vspace{-0.2cm}

%\subsection{Lasso Weighted Power $k$-means }
%\begin{equation}\label{lwpk}
%f_s(\Theta,\bw) = \sum_{i=1}^n M_s(\|\bx-\btheta_1\|_{\bw^2+\lambda \bw}^2,\dots,\|\bx-\btheta_k\|_{\bw^2+\lambda \bw}^2) ,
%\end{equation}
%By a similar argument as in Section \ref{erpk}, the MM steps for minimizing Eqn. \ref{lwpk} amounts to minimizing,
%\begin{equation}\label{lwpk2}
%\sum_{i=1}^n \sum_{j=1}^k \phi_{ij}^{(m)} \|\bx_i-\btheta_{j}\|_{\bw}^2+\sum_{l=1}^p w_l \sum_{i=1}^n \sum_{j=1}^k \phi_{ij}^{(m)} (x_{il}-\theta_{jl})^2
%\end{equation}
%Minimizing eqn. \ref{lwpk2} w.r.t. the constrain $\sum_{l=1}^pw_l=1$ can be solved the following updates.
% \begin{align}
% \btheta_{m+1,j} & =\frac{\sum_{i=1}^n \phi_{ij}^{(m)} \bx_i}{\sum_{i=1}^n \phi_{ij}^{(m)}}\\
% w_{m+1,l} & =S\bigg(\frac{\alpha^*}{\sum_{i=1}^n \sum_{j=1}^k \phi_{ij}^{(m)}(x_{il}-\theta_{jl})^2},\lambda\bigg),
% \end{align}
% where $\alpha^*$ solves the equation 
% \begin{equation}
% \sum_{l=1}^p S\bigg(\frac{\alpha^*}{\sum_{i=1}^n \sum_{j=1}^k \phi_{ij}^{(m)}(x_{il}-\theta_{jl})^2},\lambda\bigg)=1
% \end{equation}
%% \begin{minipage}{0.48\textwidth}
 \begin{algorithm}
 \KwData{$\mathbf{X}\in \mathbb{R}^{n \times p}$, $\lambda>0$, $\eta>1$}
 \KwResult{$\bTheta$}
 initialize $s_0<0$ and $\bTheta_0$\\
 \textbf{repeat}:
 {\small
 \begin{align*}
 &\phi_{ij}^{(m)}  \leftarrow\frac{1}{k}\|\bx_i-\btheta_{m,j}\|_{\bw_m}^{2(s_m-1)}
  \bigg(\frac{1}{k}\sum_{j=1}^k\|\bx_i-\btheta_{m,j}\|_{\bw_m}^{2s_m}\bigg)^{(\frac{1}{s_m}-1)} \\
 &\btheta_{m+1,j}  \leftarrow \frac{\sum_{i=1}^n \phi_{ij}^{(m)} \bx_i}{\sum_{i=1}^n \phi_{ij}^{(m)}}\\
 &w_{m+1,l} \leftarrow \frac{\exp{\bigg\{-\frac{\sum_{i=1}^n \sum_{j=1}^k \phi_{ij}^{(m)}(x_{il}-\theta_{jl})^2}{\lambda}\bigg\}}}{\sum_{t=1}^p \exp{\bigg\{-\frac{\sum_{i=1}^n \sum_{j=1}^k \phi_{ij}^{(m)}(x_{it}-\theta_{jt})^2}{\lambda}\bigg\}}}\\
 & s_{m+1} \leftarrow \eta s_m 
 \end{align*}
 }%
 \textbf{until} convergence
 \caption{Entropy Weighted Power $k$-means Algorithm (EWP)}
 \label{alg}
\end{algorithm}
%\begin{algorithm}
% \KwData{$\mathbf{X}\in \mathbb{R}^{n \times p}$, $\lambda>0$}
% \KwResult{$\bTheta$}
% initialize $s_0<0$ and $\bTheta_0$\\
% \textbf{repeat}:
% \begin{align*}
% \phi_{ij}^{(m)} & \leftarrow\frac{1}{k}\|\bx_i-\btheta_{m,j}\|_{\bw_m}^{2(s_m-1)} \bigg(\frac{1}{k}\sum_{j=1}^k\|\bx_i-\btheta_{m,j}\|_{\bw_m}^{2s_m}\bigg)^{(\frac{1}{s_m}-1)}    
% \end{align*}
% $$\btheta_{m+1,j}  \leftarrow \frac{\sum_{i=1}^n \phi_{ij}^{(m)} \bx_i}{\sum_{i=1}^n \phi_{ij}^{(m)}}$$
% Find $\alpha^*$ such that $ \sum_{l=1}^p S\bigg(\frac{\alpha^*}{\sum_{i=1}^n \sum_{j=1}^k \phi_{ij}^{(m)}(x_{il}-\theta_{jl})^2},\lambda\bigg)=1$ by bisection.
% $$w_{m+1,l}  =S\bigg(\frac{\alpha^*}{\sum_{i=1}^n \sum_{j=1}^k \phi_{ij}^{(m)}(x_{il}-\theta_{jl})^2},\lambda\bigg)$$
% $$s_{m+1}=\eta s_m $$
% \textbf{until} convergence
% \caption{Entropy Regularized Power $k$-means algorithm}
%\end{algorithm}
%%\vspace{-0.5cm}
\section{Theoretical Properties}%vspace{-0.1cm}
\label{theory}
We note that all iterates $\btheta_m$ in Algorithm 1 are defined within the convex hull of the data, all weight updates lie within $[0,1]$, and the procedure enjoys convergence guarantees as an MM algorithm \citep{lange2016mm}.
Before we state and prove the main result of this section on strong consistency, we present results characterizing the sequence of minimizers. Theorems \ref{hull} and \ref{uniform} show that the minimizers of surfaces $f_s$ always lie in the convex hull of the data $C^k$, and converge uniformly to the minimizer of $f_{-\infty}$.  %which says that $\bTheta_{n,s}$ lies inside $C^k$. A detailed proof of this result can be found in the supplementary document.
\begin{thm}\label{hull}
Let $s \leq 1$ also let $(\bTheta_{n,s},\bw_{n,s})$ be minimizer of $f_s(\bTheta,\bw)$. Then we have $\bTheta_{n,s} \in C^k$.
\end{thm}

\begin{proof}
Let $P_C^{\bw} (\btheta)$ denote the projection of $\btheta$ onto $C$ w.r.t. the $\|\cdot\|_{\bw}$ norm. Now for any $\bv \in C$, using the obtise angle condition, we obtain, $\langle \btheta-P_C^{\bw} (\btheta), \bv-P_C^{\bw} (\btheta) \rangle_{\bw} \leq 0$. Since $\bx_i \in C$, we obtain,
\begin{align*}
    \|\bx_i-\btheta_j\|^2_{\bw} & = \|\bx_i-P_C^{\bw} (\btheta_j)\|^2_{\bw} + \|P_C^{\bw} (\btheta_j)-\btheta_j\|^2_{\bw}\\
    &-2 \langle \btheta-P_C^{\bw} (\btheta_j), \bx_i-P_C^{\bw} (\btheta_j) \rangle_{\bw}\\
    & \geq \|\bx_i-P_C^{\bw} (\btheta_j)\|^2_{\bw} + \|P_C^{\bw} (\btheta_j)-\btheta_j\|^2_{\bw}.
\end{align*}
Now since, $M_s(\cdot)$ is an increasing function in each of its argument, if we replace $\btheta_j$ by $P_C^{\bw} (\btheta_j)$ in $M_s(\|\bx_i-\btheta_1\|^2_{\bw},\dots,\|\bx_i-\btheta_k\|^2_{\bw})$, the objective function value doesn't go up. Thus we can effectively restrict our attention to $C^k$. Now since the function $f_s(\cdot,\cdot)$ is continuous on the compact set $C^k \times [0,1]^p$, it attains its minimum on $C^k \times [0,1]^p$. Thus, $\bTheta^* \in C^k$. 
\end{proof}
%The following theorem states the uniform convergence of the sequence of objective functions, which are decreasing in $s$. Interested readers can find the proof of Theorem \ref{uniform} in the supplementary document.
\begin{thm}
\label{uniform}
For any decreasing sequence $\{s_m\}_{m=1}^\infty$ such that $s_1 \leq 1$ and $s_m \to -\infty$, $f_{s_m}(\bTheta,\bw)$ converges uniformly to $f_{-\infty}(\bTheta,\bw)$ on $C^k\times [0,1]^p$.
\end{thm}
\begin{proof}
For any $(\bTheta,\bw) \in C^k\times [0,1]^p$, $f_{s_m}(\bTheta,\bw)$ converges monotonically to $f_{-\infty}(\bTheta,\bw)$ (this is due to the power mean inequality). Since $C^k\times [0,1]^p$ is compact, the result follows immediately upon applying Dini's theorem from real analysis.
\end{proof}

Strong consistency is a fundamental requirement of any ``good" estimator in the statistical sense: as the number of data points grows, one should be able to recover true parameters with arbitrary precision \citep{terada2014strong,terada2015strong,chakraborty2019strong}.
%Consistency results for clustering was first introduced by \cite{pollard1981strong} for $k$-means clustering and subsequent extensions to other clustering algorithms can be found in the works of \cite{terada2014strong,terada2015strong,chakraborty2019strong}. 
The proof of our main result builds upon the core argument for $k$-means consistency by \cite{pollard1981strong}, and extends the argument through novel arguments involving uniform convergence of the family of annealing functions. \par 
Let $\bx_1,\dots,\bx_n \in \mathbb{R}^p$ be independently and identically distributed from distribution $P$ with support on a compact set $C \subset \mathbb{R}^p$. For notational convenience, we write $\mathcal{M}_s(\bx,\bTheta,\bw)$ for $M_s(\|\bx-\btheta_1\|_{\bw},\dots,\|\bx-\btheta_1\|_{\bw})$. We consider the following minimization problem 
$$\min_{\bTheta,\bw}\bigg\{\frac{1}{n} \sum_{i=1}^n \mathcal{M}_s(\bx_i,\bTheta,\bw)+\lambda \sum_{l=1}^p w_l \log w_l\bigg\},$$
which is nothing but a scaled version of equation \eqref{obj}. Intuitively, as $n \to \infty$, $\frac{1}{n} \sum_{i=1}^n \mathcal{M}_s(\bx_i,\bTheta,\bw)$ is very close to $\int  \mathcal{M}_s(\bx,\bTheta,\bw)dP$ almost surely by appealing to the Strong Law of Large Numbers (SLLN).
Together with \eqref{eq:limit}, as $n \to \infty$ and $s \to -\infty$ we expect
%Now as $s \to -\infty$, $\mathcal{M}_s(\bx,\bTheta,\bw)$ approximates $\min_{\btheta \in \bTheta} \|\bx-\btheta\|_{\bw}$. Thus as $n \to \infty$ and $s \to -\infty$, we expect 
\begin{equation}\label{former}\frac{1}{n} \sum_{i=1}^n \mathcal{M}_s(\bx_i,\bTheta,\bw)+\lambda \sum_{l=1}^p w_l \log w_l\end{equation} 
to be in close proximity of 
\begin{equation}\label{latter}\int \min_{\btheta \in \bTheta} \|\bx-\btheta\|_{\bw}dP++\lambda \sum_{l=1}^p w_l \log w_l, \end{equation} so that minimizers of the \eqref{former} should be very close to the minimizers of \eqref{latter} under certain regularity conditions.\par
To formalize this intuition, let $\bTheta^*$, $\bw^*$ be minimizers of $$\Phi(\bTheta,\bw)~=\int \min_{1\leq j \leq k} \|\bx-\btheta_j\|^2_{\bw}dP+\lambda \sum_{l=1}^p w_l\log w_,$$ and define $\bTheta_{n,s}$, $\bw_{n,s}$ as the minimizers of $$\int \mathcal{M}_s(\bx,\bTheta,\bw)dP_n+\lambda  \sum_{l=1}^p w_l\log w_l ,$$ where $P_n$ is the empirical measure. We will show that $\bTheta_{n,s} \xrightarrow{a.s.} \bTheta^* $ and $\bw_{n,s} \xrightarrow{a.s.} \bw^*$ as $n \to \infty$ and $s \to -\infty$ under the following identifiability assumption:
%To show this result, we will assume that the minimizer $\bTheta^*$ and $\bw^*$ exist uniquely. In particular we assume the following.
%\vspace{-0.2cm}
\begin{itemize}
\item[A1] For any neighbourhood $N$ of $(\bTheta^*,\bw^*)$, there exists $\eta>0$ such that if $(\bTheta, \bw) \not \in N$ implies that $\Phi(\bTheta, \bw)> \Phi(\bTheta^*, \bw^*)+\eta$. 
\end{itemize}
%\vspace{-0.2cm}

% \begin{proof}
% For any $(\bTheta,\bw) \in C^k\times [0,1]^p$, $f_{s_m}(\bTheta,\bw)$ converges monotonically to $f_{-\infty}(\bTheta,\bw)$ (this is due to the power mean inequality). Since $C^k\times [0,1]^p$ is compact, applying Dini's theorem \citep{bartle2011introduction}, the result follows.
% \end{proof}
Theorem \ref{slln} establishes a uniform SLLN, which plays a key role in the proof of the main result (Theorem \ref{main theorem}).
\begin{thm}
\label{slln}
(SLLN) Fix $s \leq 1$. Let $\mathcal{G}$ denote the family of functions $g_{\bTheta,\bw}(\bx)=\mathcal{M}_s(\bx,\bTheta,\bw)$. Then $\sup_{g \in \mathcal{G}}|\int g dP_n -\int g dP|\to 0$ a.s. $[P]$.
\end{thm}
\begin{proof}
Fix $\epsilon >0$. It is enough to find a finite family of functions $\mathcal{G}_\epsilon$ such that for all $g \in \mathcal{G}$, there exists $\Bar{g},\Dot{g} \in \mathcal{G}_\epsilon$ such that $\Dot{g} \leq g \leq  \Bar{g}$ and $\int(\Bar{g}-\Dot{g})dP<\epsilon$.\par
Let us define $\phi(\cdot): \mathbb{R} \to \mathbb{R}$ such that $\phi(x)=\max\{0,x\}$. Since $C$ is compact, for every $\delta_1>0$, we can always construct a finite set $C_{\delta_1} \subset C$ such that if $\btheta \in C$, there exist $\btheta' \in C_{\delta_1}$ such that $\|\btheta-\btheta'\|<\delta_1$. Similarly, resorting to the compactness of $[0,1]^p$, for every $\delta_2>0$, we can always construct a finite set $W_{\delta_2} \subset [0,1]^p$ such that if $\bw \in [0,1]^p$, there exist $\bw' \in W_{\delta_2}$ such that $\|\bw-\bw'\|<\delta_2$. Consider the function $h(\bx,\bTheta,\bw)=M_s(\|\bx-\btheta_1\|^2_{\bw},\dots,\|\bx-\btheta_k\|^2_{\bw})$ on $C\times C^k \times [0,1]^p$. $h$, being continuous on the compact set $C\times C^k \times [0,1]^p$, is also uniformly continuous. Thus for all $\mathbf{x} \in C$, if $\|\bw-\bw'\|<\delta_2$ and $\|\btheta_j-\btheta'_j\|<\delta_1$ for all $j=1,\dots,k$ implies that 
\begin{align}
\label{eqq1}
\bigg|M_s(\|\bx-\btheta_1\|^2_{\bw},\dots,\|\bx-\btheta_k\|^2_{\bw}) -M_s(\|\bx-\btheta_1'\|^2_{\bw'},\dots,\|\bx-\btheta'_k\|^2_{\bw'})\bigg|< \epsilon/2
\end{align}

We take 
\begin{align*}
\mathcal{G}_\epsilon & =\{\phi(M_s(\|\bx-\btheta_1'\|^2_{\bw'},\dots,\|\bx-\btheta'_k\|^2_{\bw'})\pm \epsilon/2)\\
&: \btheta'_1,\dots,\btheta'_k \in C_{\delta_1} \text{ and } \bw' \in   W_{\delta_2}\}. 
\end{align*}
Now if we take 

$$\Bar{g}_{\btheta,\bw}=\phi(M_s(\|\bx-\btheta_1'\|^2_{\bw'},\dots,\|\bx-\btheta'_k\|^2_{\bw'})+ \epsilon/2)$$

and 
$$\Dot{g}_{\btheta,\bw}=\phi(M_s(\|\bx-\btheta_1'\|^2_{\bw'},\dots,\|\bx-\btheta'_k\|^2_{\bw'})- \epsilon/2),$$
where $\btheta'_j \in C_{\delta_1}$ and $\bw \in W_{\delta_2}$ for $j=1,\dots,k$ such that $\|\btheta_j-\btheta'_j\|<\delta_1$ and $\|\bw-\bw'\|<\delta_2$. From equation (\ref{eqq1}), we get, $\Dot{g} \leq g \leq \Bar{g}$. Now we need to show $\int(\Bar{g}-\Dot{g})dP<\epsilon$. This step is straight forward. 
\begin{align*}
&\int(\Bar{g}-\Dot{g})dP\\
&= \bigg[ \phi(M_s(\|\bx-\btheta_1'\|^2_{\bw'},\dots,\|\bx-\btheta'_k\|^2_{\bw'})+ \epsilon/2)\\
& - \phi(M_s(\|\bx-\btheta_1'\|^2_{\bw'},\dots,\|\bx-\btheta'_k\|^2_{\bw'})- \epsilon/2)\bigg] dP\\
& \leq \epsilon \int dP = \epsilon.
\end{align*}
 Hence the result.
\end{proof}

We are now ready to establish the main consistency result, stated and proven below.
\begin{thm}
\label{main theorem}
Under the condition A1, $\bTheta_{n,s} \xrightarrow{a.s.} \bTheta^* $ and $\bw_{n,s} \xrightarrow{a.s.} \bw^*$ as $n \to \infty$ and $s \to -\infty$.
\end{thm} 
\begin{proof}
It is enough to show that given any neighbourhood $N$ of $(\bTheta^*,\bw^*)$, there exists $M_1<0$ and $M_2>0$ such that if $s<M_1$ and $n>M_2$ such that $(\bTheta, \bw) \in N$ almost surely. By assumption A1, it is enough to show that for all $\eta>0$, there exists $M_1<0$ and $M_2>0$ such that if $s<M_1$ and $n>M_2$ such that $\Phi(\bTheta, \bw)\leq  \Phi(\bTheta^*, \bw^*)+\eta$ almost surely. For notational convenience, we write $\mathcal{M}_s(\bx,\bTheta,\bw$) for  $M_s(\|\bx-\btheta_1\|^2_{\bw},\dots,\|\bx-\btheta_k\|^2_{\bw})$ and $\alpha(\bw)=\lambda \sum_{l=1}^pw_l \log w_l$. Now since $(\bTheta_{n,s},\bw_{n,s})$ is the minimizer for $\int \mathcal{M}_s(\bx,\bTheta,\bw)dP_n + \lambda \sum_{l=1}^pw_l \log w_l$, we get,
\begin{align}
\label{eqq2}
&\int \mathcal{M}_s(\bx,\bTheta_{n,s},\bw_{n,s})dP_n +\lambda \alpha(\bw_{n,s}) \nonumber \\
& \leq \int \mathcal{M}_s(\bx,\bTheta^*,\bw^*)dP_n + \lambda \alpha(\bw^*).
\end{align}
Now observe that $\Phi(\bTheta_{n,s},\bw_{n,s})-\Phi(\bTheta^*,\bw^*)=\xi_1+\xi_2+\xi_3$, where,
\begin{align*}
\xi_1 & =\Phi(\bTheta_{n,s},\bw_{n,s})-\int \mathcal{M}_s(\bx,\bTheta_{n,s},\bw_{n,s})dP-\lambda \alpha(\bw_{n,s}),\\
\xi_2 & =  \int \mathcal{M}_s(\bx,\bTheta_{n,s},\bw_{n,s})dP-\int \mathcal{M}_s(\bx,\bTheta_{n,s},\bw_{n,s})dP_n,\\
\xi_3 & = \int \mathcal{M}_s(\bx,\bTheta_{n,s},\bw_{n,s})dP_n+\lambda \alpha(\bw_{n,s})- \Phi(\bTheta^*,\bw^*).
\end{align*} 
We first choose $M_1<0$ such that if $s<M_1$ then 
\begin{equation}
    \label{eqq3}
    \bigg|\min_{1 \leq j \leq k} \|\bx-\btheta_j\|_{\bw}-\mathcal{M}_s(\bx,\bTheta,\bw)\bigg|<\eta/6
\end{equation}
 for all $\bx \in C$, $\bTheta \in C^k$ and $\bw \in [0,1]^p$. Thus for $s<M_1$, $\min_{1 \leq j \leq k} \|\bx-\btheta_j\|_{\bw}\leq \mathcal{M}_s(\bx,\bTheta,\bw)+\eta/6$ which in turn implies that $ \int\min_{1 \leq j \leq k} \|\bx-\btheta_j\|_{\bw}dP_n\leq \int\mathcal{M}_s(\bx,\bTheta,\bw)dP_n+\eta/3$. Substituting $\bTheta_{n,s}$ for $\Theta$ and $\bw_{n,s}$ for $\bw$ in the above expression and adding $\lambda \alpha(\bw_{n,s})$ to both sides, we get $\xi_1<\eta/6$. We also observe that the quantity $\xi_2$ can also be made smaller that $\eta/3$ by appealing to the uniform SLLN (Theorem 3). Now to bound $\xi_3$, we observe that 
\begin{align*}
\xi_3 & \leq \int \mathcal{M}_s(\bx,\bTheta^*,\bw^*)dP_n+\lambda \alpha(\bw^*)- \Phi(\bTheta^*,\bw^*)\\
& = \int \mathcal{M}_s(\bx,\bTheta^*,\bw^*)dP_n - \int \min_{\btheta \in \bTheta^*} \|\bx-\btheta\|_{\bw^*}dP
\end{align*} 
This inequality is obtained by appealing to equation (\ref{eqq2}). Again appealing to the uniform SLLN, we get that for large enough $n$,
\begingroup
\allowdisplaybreaks
\begin{align*}
    \xi_3 & \leq \int \mathcal{M}_s(\bx,\bTheta^*,\bw^*)dP - \int \min_{\btheta \in \bTheta^*}\|\bx-\btheta\|_{\bw^*}dP+\eta/6\\
    & \leq \int [\min_{\btheta \in \bTheta^*}\|\bx-\btheta\|_{\bw^*}+\eta/6]dP - \int \min_{\btheta \in \bTheta^*}\|\bx-\btheta\|_{\bw^*}dP\\
    & + \eta/6= \eta/3.
\end{align*}
\endgroup
The second inequality follows from equation (\ref{eqq3}). Thus we get, $\Phi(\bTheta_{n,s},\bw_{n,s})-\Phi(\bTheta^*,\bw^*)=\xi_1+\xi_2+\xi_3 \leq \eta/6+\eta/3+\eta/3 < \eta$ almost surely. The result now follows. 
\end{proof}
\section{Empirical Performance}
%\vspace{-0.1cm}
\label{experiment}
We examine the performance of EWP on a variety of simulated and real datasets compared to classical and state-of-the-art peer algorithms. All the datasets and the codes pertaining to this paper are publicly available at \url{https://github.com/DebolinaPaul/EWP}. For evaluation purposes, we use the Normalized Mutual Information (NMI) \citep{vinh2010information} between the ground-truth partition and the partition obtained by each algorithm. A value of 1 indicates perfect clustering and a value of 0 indicates arbitrary labels. As our algorithm is meant to perform as a drop-in replacement to $k$-means, we focus comparisons to Lloyd's classic algorithm \citep{lloyd1982least}, $WK$-means \citep{huang2005automated}, Power $k$-means \citep{xu2019power} and sparse $k$-means \citep{witten2010framework}. It should be noted that sparse $k$-means already entails higher computational complexity, and we do not exhaustively consider alternate methods which require orders of magnitude of higher complexity. In all cases, each algorithm is  initiated with the same set of randomly chosen centroids.%%\vspace{-0.4cm} % to ensure that the difference in the final clustering is solely due to the difference in the algorithms and {\it not} due the having a better initialization.
\subsection{Synthetic Experiments}
%\vspace{-0.1cm}
\label{simulation}
We now consider a suite of simulation studies to validate the proposed EWP  algorithm. %\vspace{-0.3cm}
\paragraph{Simulation 1}
\label{simulation 1}
\begin{figure}[ht]
    \centering
    \hspace{-17pt}
     \begin{subfigure}[b]{0.34\textwidth}
    \centering
        \includegraphics[height=\textwidth,width=\textwidth]{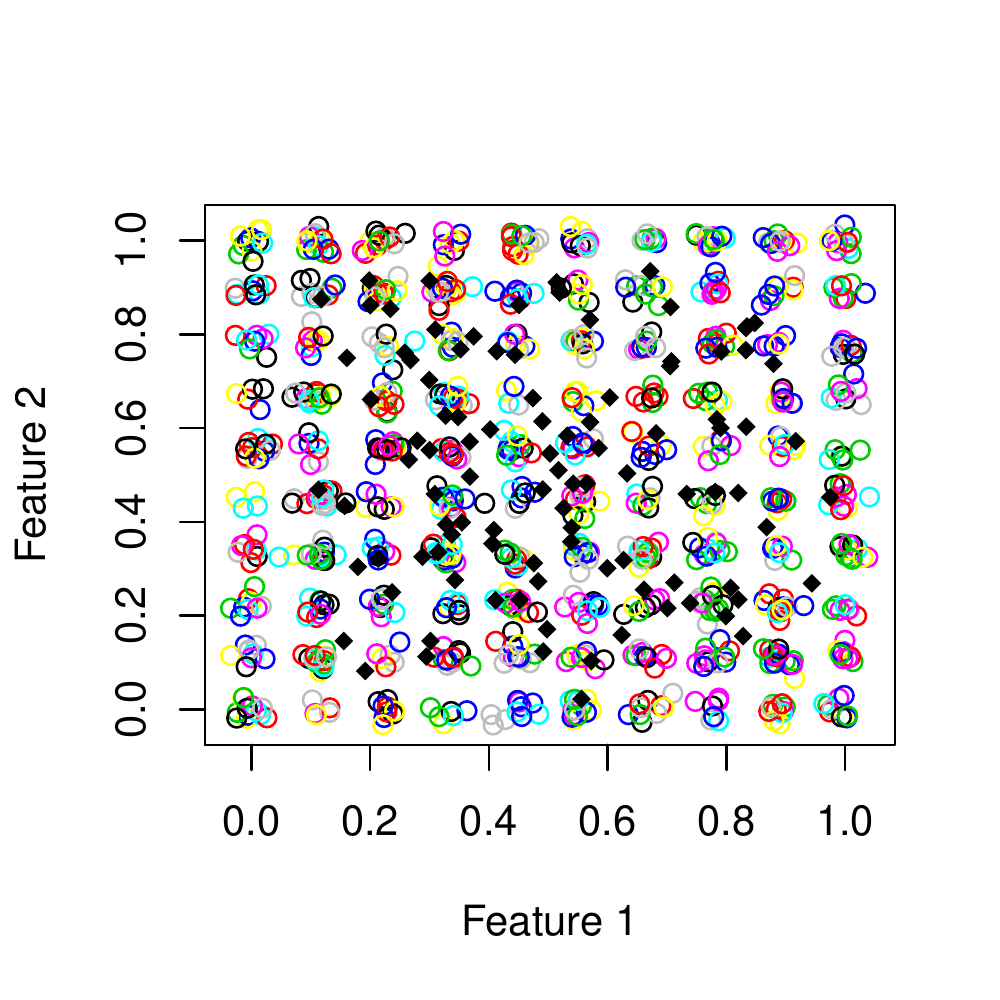}
        \caption{$k$-means}
        %\label{gt}
    \end{subfigure}
    ~
    \hspace{-17pt}
 \begin{subfigure}[b]{0.34\textwidth}
    \centering
        \includegraphics[height=\textwidth,width=\textwidth]{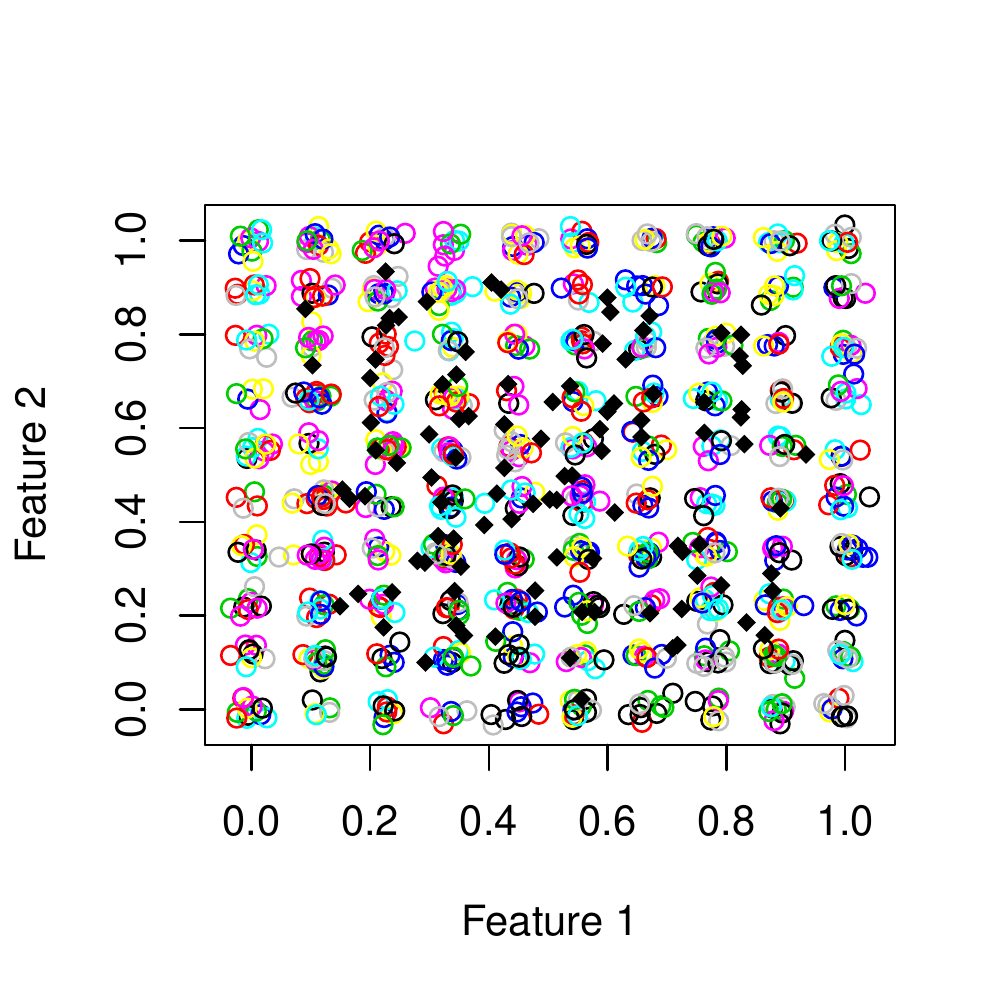}
        \caption{$WK$-means}
        %\label{hist}
    \end{subfigure}
    ~
    \hspace{-17pt}
     \begin{subfigure}[b]{0.34\textwidth}
    \centering
        \includegraphics[height=\textwidth,width=\textwidth]{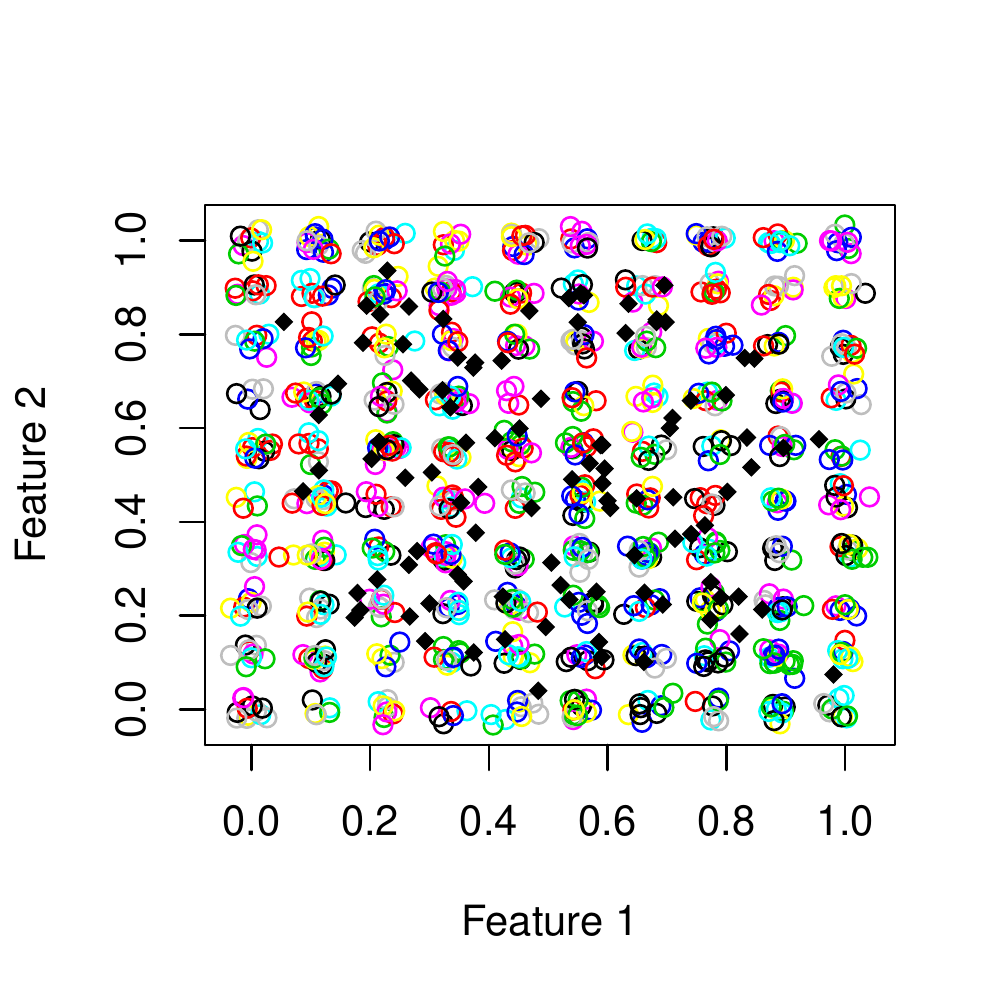}
        \caption{Power $k$-means}
        %\label{hist}
    \end{subfigure}
    ~
    \hspace{-17pt}
     \begin{subfigure}[b]{0.34\textwidth}
    \centering
        \includegraphics[height=\textwidth,width=\textwidth]{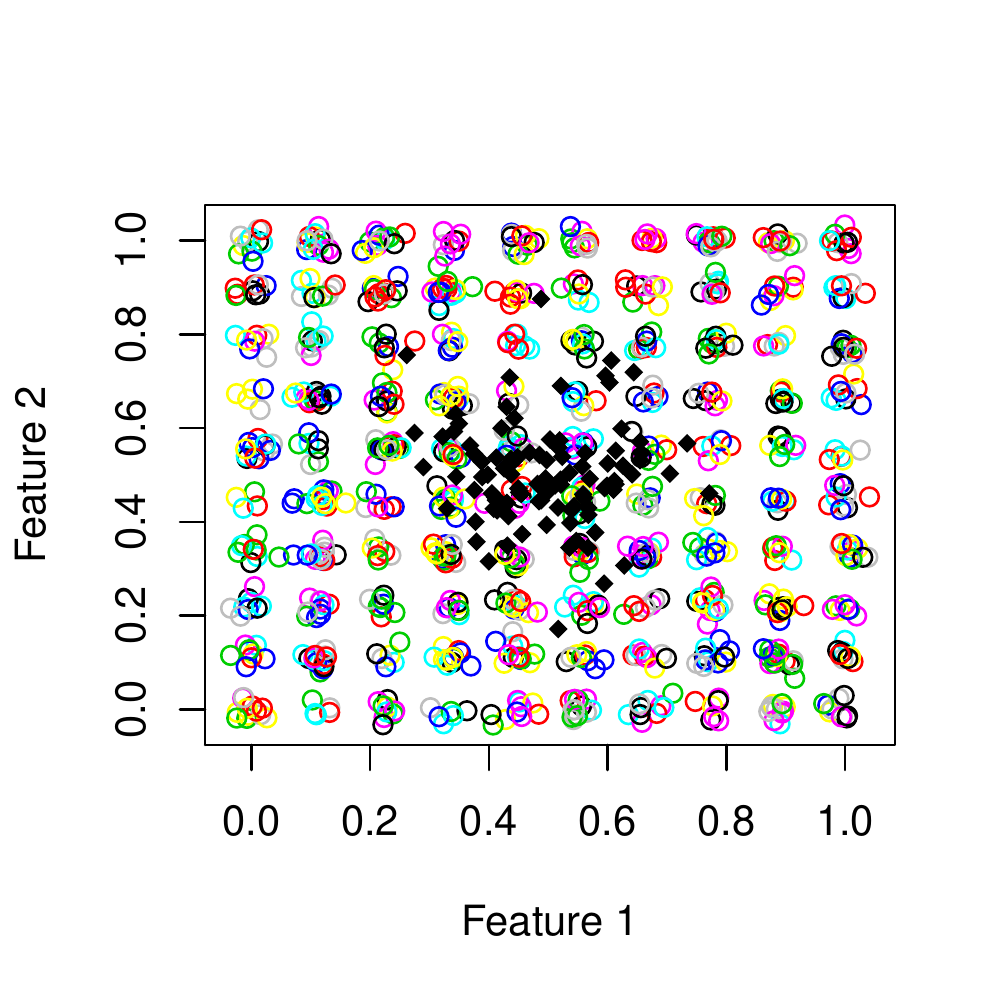}
        \caption{Sparse $k$-means}
       % \label{hist}
    \end{subfigure}
    ~
    \hspace{-17pt}
 \begin{subfigure}[b]{0.34\textwidth}
    \centering
        \includegraphics[height=\textwidth,width=\textwidth]{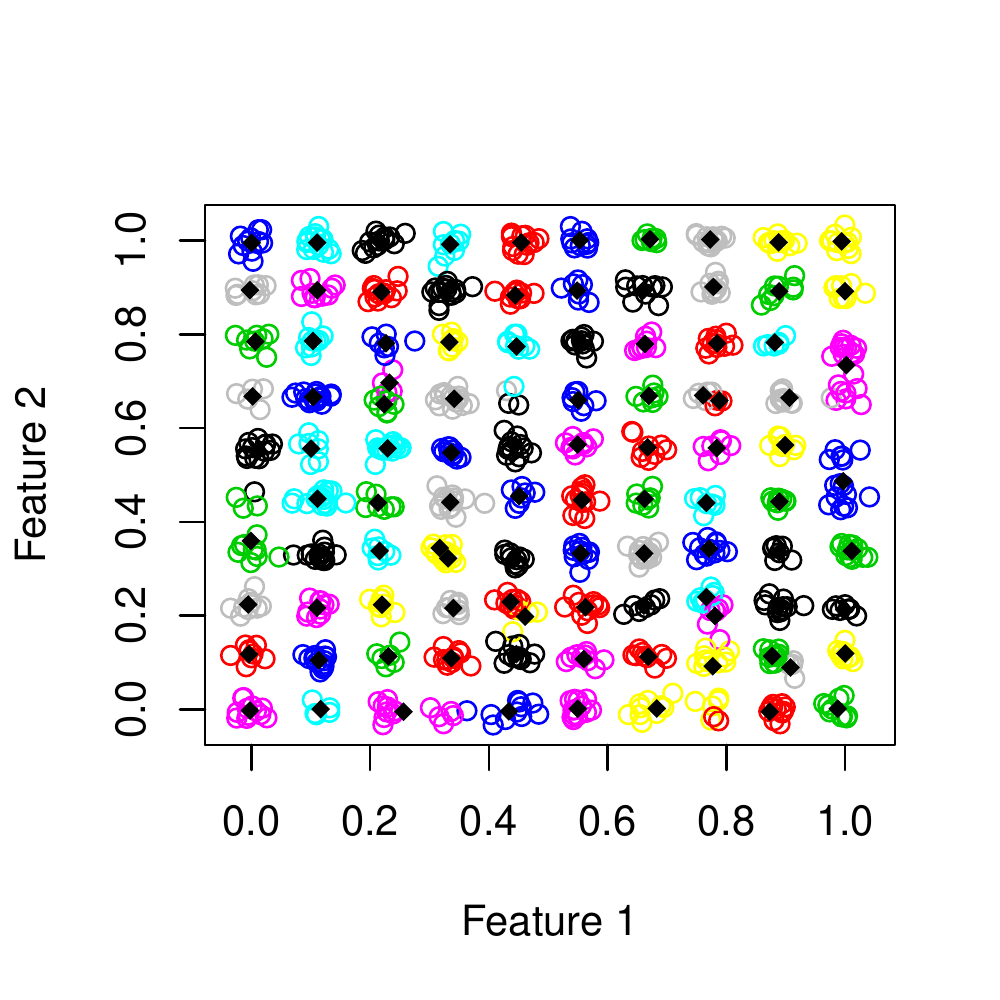}
        \caption{EWP}
        %\label{hist}
    \end{subfigure}
            \caption{Solutions obtained by the peer algorithms for an example dataset with $k=100$ and $d=20$ in Simulation 1 (\ref{simulation 1}). The obtained cluster centroids appear as black diamonds in the figure. }%Best viewed in color when zoomed in the electronic version.}
        \label{fig: sim1}
\end{figure}
The first experiment assesses performance as the dimension and number of uninformative features grows. We generate $n=1000$ observations with $k=100$ clusters. Each observation has $p=d+2$ many features as $d$ varies between $5$ and $100$. The first two features reveal cluster structure, while the remaining $d$ variables %do not contain any information about the cluster structure of the dataset. These $d$ features 
are uninformative, generated independently from a $Unif(0,2)$ distribution. True centroids are spaced uniformly on a grid with $\theta_m=\frac{m-1}{10}$, and $x_{ij} \sim \frac{1}{10}\sum_{m=1}^{10} \mathcal{N}(\theta_m,0.15)$. Despite the simple data generating setup, clustering is difficult due to the low signal to noise ratio in this setting. \par
We report the average NMI values between the ground-truth partition and the partition obtained by each of the algorithms over $20$ trials in Table \ref{tab:s1}, with the standard deviations appearing in parentheses. 
The best performing algorithm in each column appears in bold, and the best solutions for $d=20$ are plotted in Figure \ref{fig: sim1}. The benefits using EWP are visually stark, and Table \ref{tab:s1} verifies in detail that EWP outperforms the classical $k$-means algorithm as well as the state-of-the-art sparse-$k$-means and the power $k$-means algorithms. The inability of $k$-means and Power $k$-means to properly learn the feature weights results in poor performance of these algorithms. On the other hand, although $WK$-means and sparse $k$-means can select features successfully, they fail from the optimization perspective when $k$ is large enough so that there are many local minima to trap the algorithm.  %\vspace{-0.3cm}
 \begin{table}
\caption{NMI values for Simulation 1, showing the effect of the number of unimportant features.}
    \label{tab:s1}
    \centering
%    \resizebox{0.48\textwidth}{!}{
    \begin{tabular}{|c|c|c|c|c|c|}
        \hline
         &$d=5$ & $d=10$ & $d=20$ & $d=50$ & $d=100$  \\
         \hline 
        $k$-means & 0.3913 (0.002) & 0.3701 (0.002)  & 0.3674 (0.003) &  0.3629(0.002) & 0.3517 (0.003) \\
        \hline
        $WK$-means & 0.5144(0.002) & 0.50446(0.003) & 0.5050(0.003) & 0.5026(0.005) & 0.5029(0.003) \\
        \hline
        Power $k$-means & 0.3924(0.001) & 0.3873(0.002) & 0.3722 (0.001) & 0.3967 (0.003) & 0.3871 (0.004)\\
        \hline
        Sparse $k$-means & 0.3679 (0.002) & 0.3677 (0.002) & 0.3668 (0.001) & 0.3675 (0.002) & 0.3637 (0.002)\\
        \hline
        EWP-$k$-means & \textbf{0.9641} (0.001) & \textbf{0.9217} (0.001) &  \textbf{0.9139} (0.001) & \textbf{0.9465} (0.001)  & \textbf{0.9082} (0.003)\\
        \hline
    \end{tabular}
 %   }
\end{table}
% \begin{table}%[h]
% \caption{NMI values for Simulation 1, showing the effect of the number of unimportant features.}
% %\vspace{-0.1cm}
%     \label{tab:s1}
%     \centering
%     \resizebox{0.49\textwidth}{!}{
%     \begin{tabular}{|c|c|c|c|c|c|}
%         \hline
%          &$d=5$ & $d=10$ & $d=20$ & $d=50$ & $d=100$  \\
%          \hline 
%         $k$-means & 0.3913 & 0.3701 & 0.3674 &  0.3629 & 0.3517 \\
%         \hline
%         $WK$-means & 0.5144 & 0.50446 & 0.5050 & 0.5026 & 0.5029 \\
%         \hline
%         Power $k$-means & 0.3924 & 0.3873 & 0.3722 & 0.3967 & 0.3871\\
%         \hline
%         Sparse $k$-means & 0.3679 & 0.3677 & 0.3668 & 0.3675 & 0.3637\\
%         \hline
%         EWP & \textbf{0.9641} & \textbf{0.9217} & \textbf{0.9139} & \textbf{0.9465} & \textbf{0.9082}\\
%         \hline
%     \end{tabular}
%     }
%     %\vspace{-0.5cm}
% \end{table}
\paragraph{Simulation 2}
\label{set2}
We next examine the effect of $k$ on the performance, taking $n=100 \cdot k$ and $p=100$ while $k$ varies from $20$ to $500$. The matrix $\Theta_{k \times p}$, whose rows contain the cluster centroids, is generated as follows.%\vspace{-0.4cm}
\begin{enumerate}
\item Select $5$ relevant features $l_1,\dots,l_5$ at random.%\vspace{-0.1cm}
\item Simulate $\theta_{j,l_m} \sim Unif(0,1)$ for all $j=1,\dots,k$ and $m=1,\dots,5$.%\vspace{-0.1cm}
\item Set $\theta_{j,l}=0$ for all $l \not\in \{l_1,\dots,l_5\}$ and all $j$.%\vspace{-0.3cm}
\end{enumerate}
After obtaining $\Theta$, $x_{il}$ is simulated as follows.
\begin{align*}
x_{il}\sim& \mathcal{N}(0,1) \text{ if }   l \not\in \{l_1,\dots,l_5\}\\
x_{il}\sim& \frac{1}{k}\sum_{j=1}^k\mathcal{N}(\theta_{j,l},0.015) \text{ if }   l \in \{l_1,\dots,l_5\}.
\end{align*}
We run each of the algorithms 20 times and report the average NMI values between the ground-truth partition an the partition obtained by each of the algorithms in Table \ref{tab:s2}; with standard errors appearing in parentheses. %The standard deviations of the NMI values corresponding to each experiment are reported in Table S2 of the supplementary document.
Table \ref{tab:s2} shows that $k$-means, $WK$-means, power $k$-means, and sparse $k$-means lead to almost the same result, while EWP outperforms all the peer algorithms for each $k$ as it narrows down the large number of features and avoids local minima from large $k$ simultaneously. 
\begin{table}
\caption{NMI values for Simulation 2, showing the effect of increasing number of clusters.}
    \label{tab:s2}
    \centering
%    \resizebox{0.48\textwidth}{!}{
    \begin{tabular}{|c|c|c|c|c|}
        \hline
         &$k=20$ &$k=100$ & $k=200$ & $k=500$  \\
         \hline 
        $k$-means & 0.0674(0.001) & 0.2502(0.021) & 0.3399 (0.031) &  0.3559 (0.014) \\
        \hline
        $WK$-means & 0.0587(0.001) & 0.2247(0.002) & 0.3584(0.018) & 0.3678(0.009)\\
        \hline
        Power $k$-means &0.0681(0.001) & 0.2785(0.001) & 0.3578 (0.002) & 0.3867(0.001) \\
        \hline
        Sparse $k$-means & 0.0679(0.001) & 0.2490(0.058) & 0.6705(0.007) & 0.3537 (0.002)\\
        \hline
        EWP-$k$-means & \textbf{0.9887}(0.001) & \textbf{0.9844} (0.002) &  \textbf{0.9756}(0.001) & \textbf{0.9908} (0.001) \\
        \hline
    \end{tabular}
 %   }
\end{table}
% \begin{table}%[h]
% %%\vspace{-0.1cm}
% \caption{NMI values for Simulation  2, showing the effect of increasing number of clusters.}
% %%\vspace{-0.1cm}
%     \label{tab:s2}
%     \centering
%   \resizebox{0.48\textwidth}{!}{
%     \begin{tabular}{|p{2.5cm}|C{1.2cm}|C{1.2cm}|C{1.2cm}|C{1.2cm}|}
%         \hline
%          Algorithm &$k=20$ &$k=100$ & $k=200$ & $k=500$  \\
%          \hline 
%         $k$-means & 0.0674& 0.2502 & 0.3399 &  0.3559  \\
%         \hline
%         $WK$-means & 0.0587 & 0.2247 & 0.3584 & 0.3678\\
%         \hline
%         Power $k$-means &0.0681 & 0.2785 & 0.3578 & 0.3867 \\
%         \hline
%         Sparse $k$-means & 0.0679& 0.2490 & 0.6705 & 0.3537 \\
%         \hline
%         EWP & \textbf{0.9887} & \textbf{0.9844} &  \textbf{0.9756} & \textbf{0.9908} \\
%         \hline
%     \end{tabular}
%     }
% %    %\vspace{cm}
% \end{table}
\paragraph{Feature Selection}
We now examine the feature weighting properties of the EWP algorithm more closely. We take $n=1000$, $p=20$ and follow the same data generation procedure  described in Simulation 2. For simplicity, in the first step of the simulation study, we select $l_i=i$ for $i=1,\dots,5$. We record the feature weights obtained by EWP, sparse $k$-means and $WK$-means over $100$ replicate datasets.  The box-plot for these 100 optimal feature weights are shown in Figure \ref{boxplot} for all the three algorithms. The proposed method successfully assigns almost all  weight to relevant features 1 through 5, even though it does not make use of a sparsity-inducing penalty. Meanwhile, feature weights assigned by sparse $k$-means do not follow any clear pattern related to informative features, even though the ground truth is sparse in relevant features. The analogous plot for $WK$-means shows even worse performance than sparse $k$-means. The study clearly illustrates the necessity of feature weighing together with annealing for successful $k$-means clustering in high dimensional settings. %This demonstrates the efficacy of the our proposed method for large number of clusters in high-dimensional data.  %\vspace{-0.3cm}
\begin{figure}
%\vspace{-0.5cm}
    \centering
    \hspace{-10pt} 
    \begin{subfigure}[b]{0.34\textwidth}
    \centering
        \includegraphics[height=\textwidth,width=\textwidth]{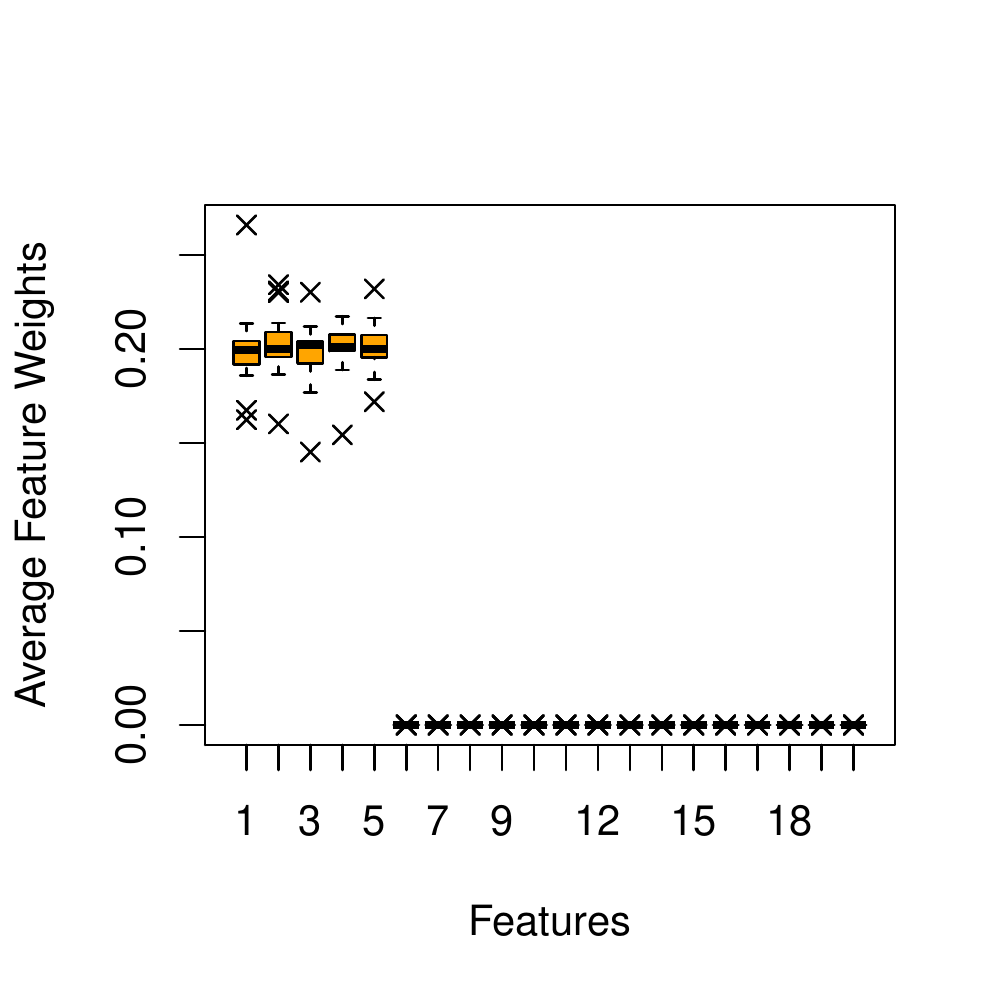}
        \caption{EWP}
        %\label{gt}
    \end{subfigure}
    ~
    \hspace{-17pt}
     \begin{subfigure}[b]{0.34\textwidth} 
    \centering
        \includegraphics[height=\textwidth,width=\textwidth]{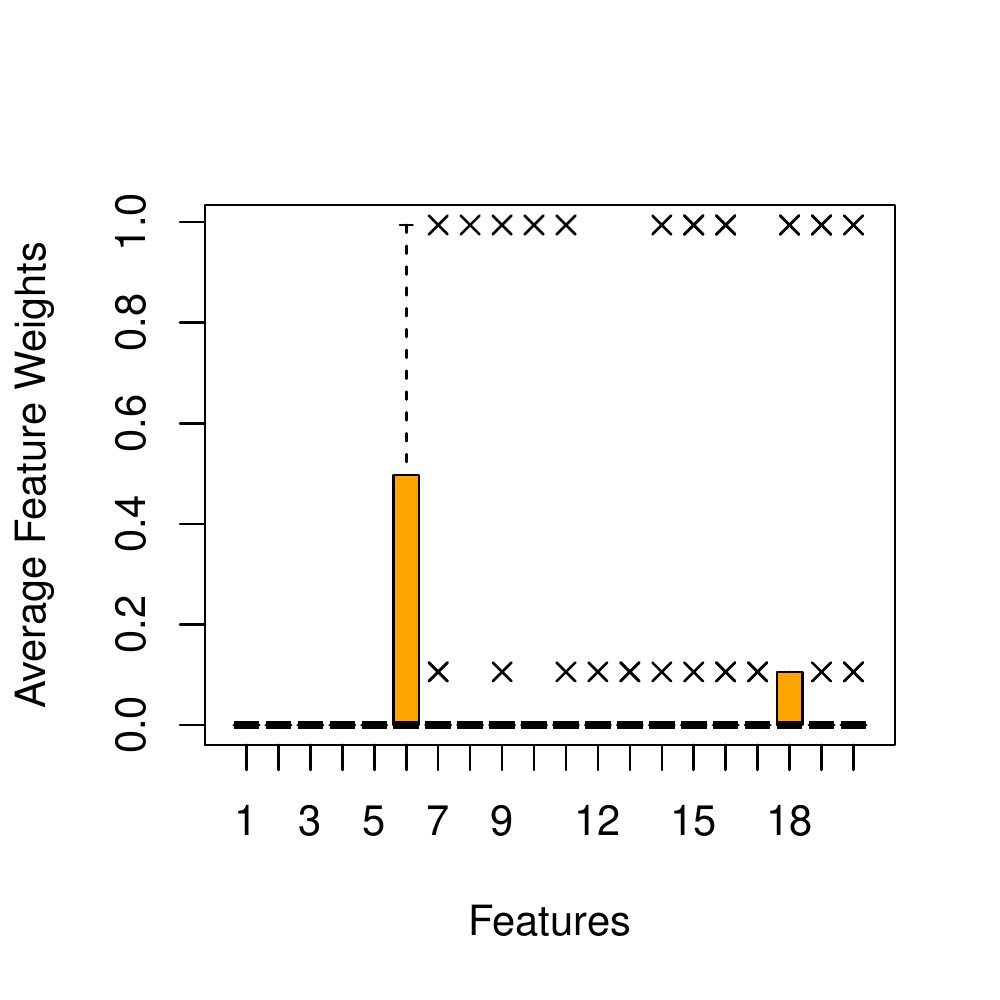}
        \caption{Sparse $k$-means}
        %\label{hist}
    \end{subfigure}
    ~
    \hspace{-17pt}
    \begin{subfigure}[b]{0.34\textwidth} 
    \centering
        \includegraphics[height=\textwidth,width=\textwidth]{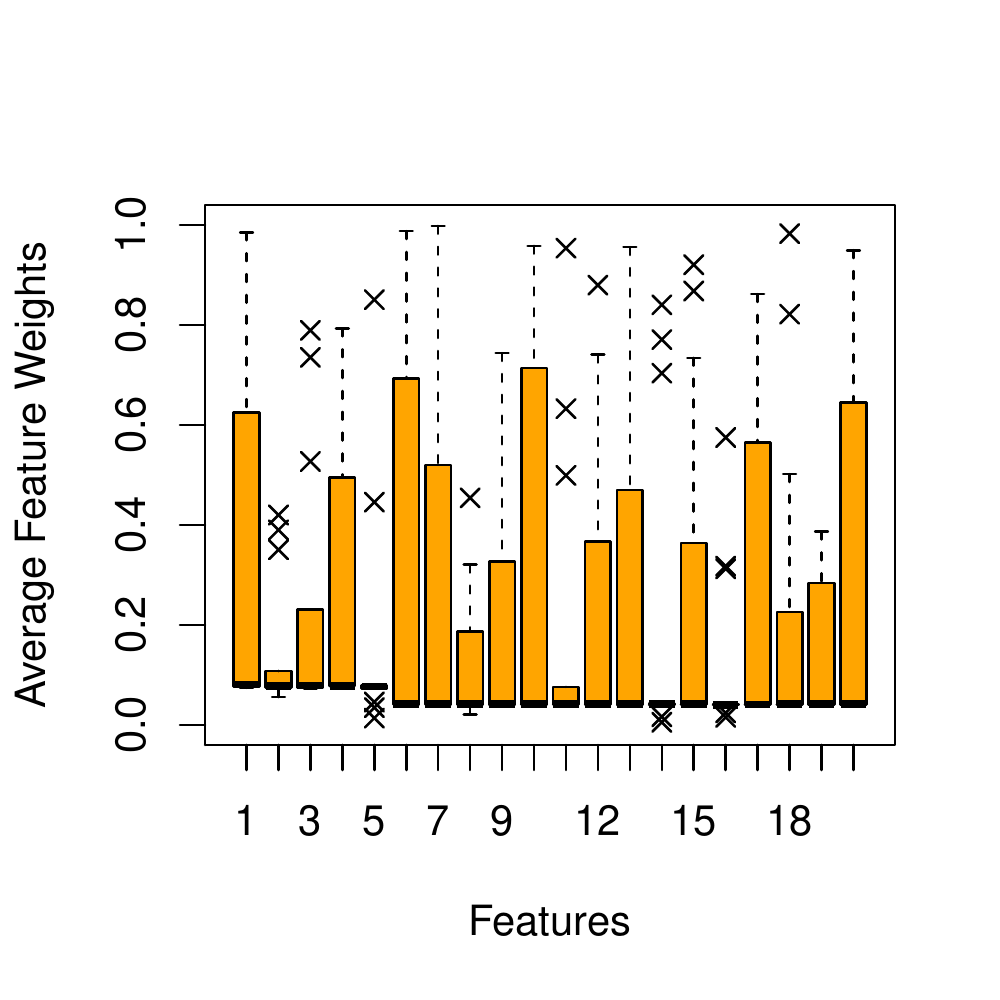}
        \caption{$WK$-means}
        %\label{hist}
    \end{subfigure}

            \caption{Boxplots show that EWP consistently identifies true features while sparse $k$-means fails to do so. }
        \label{boxplot}
        %\vspace{-2pt}
\end{figure}

\begin{figure}[ht!]
    \centering
    \hspace{-12pt} %\vspace{-15pt} 
    \begin{subfigure}[b]{0.34\textwidth}
    \centering
        \includegraphics[height=\textwidth,width=\textwidth]{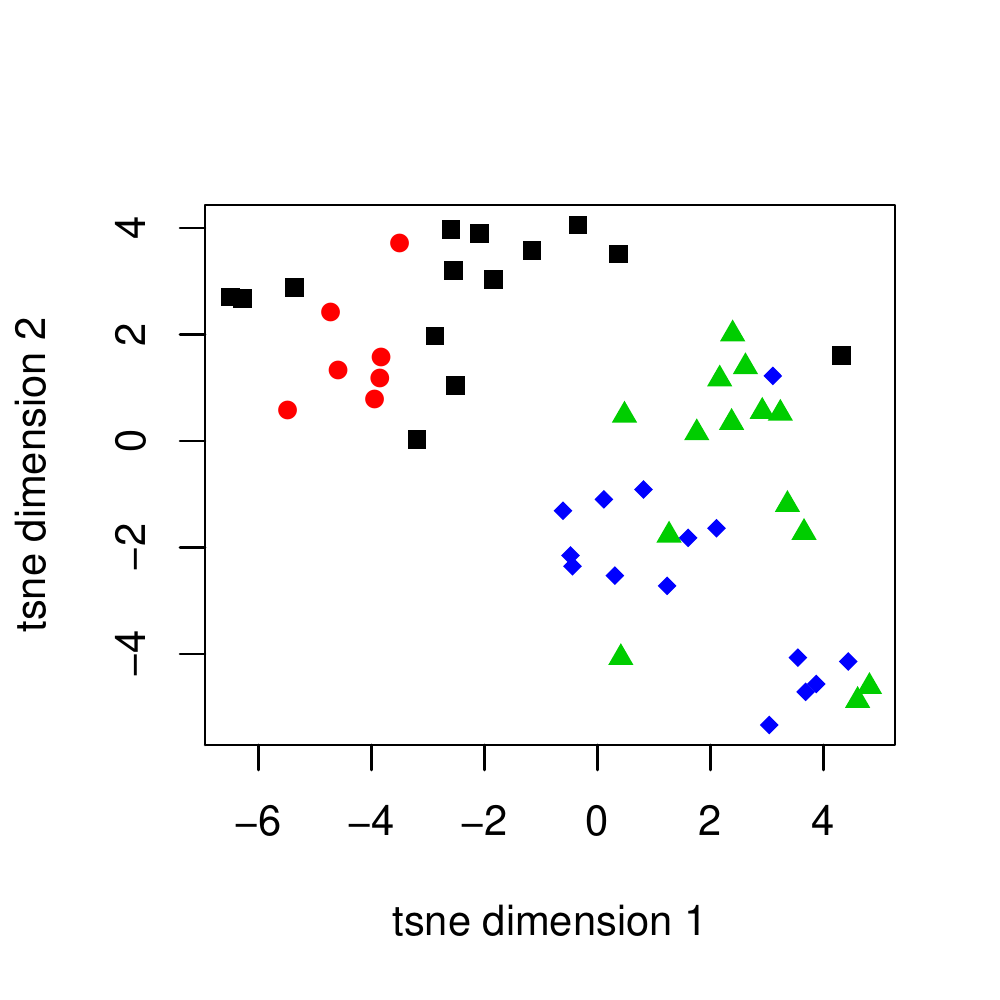}
        \caption{Ground Truth}
        %\label{gt}
    \end{subfigure}
    ~
    \hspace{-17pt}\begin{subfigure}[b]{0.34\textwidth}
    \centering
        \includegraphics[height=\textwidth,width=\textwidth]{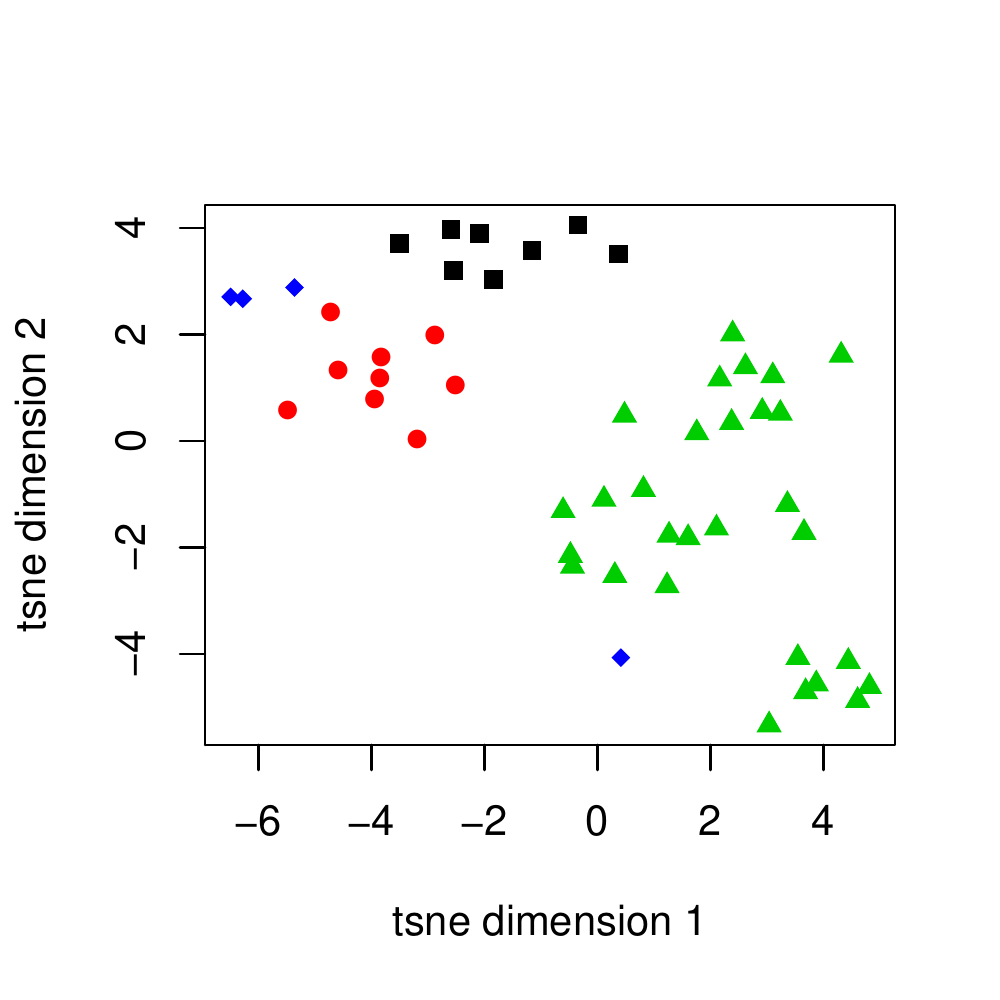}
        \caption{$k$-means}
        %\label{hist}
    \end{subfigure}
    \hspace{-12pt} %\vspace{-15pt} 
    \begin{subfigure}[b]{0.34\textwidth}
    \centering
        \includegraphics[height=\textwidth,width=\textwidth]{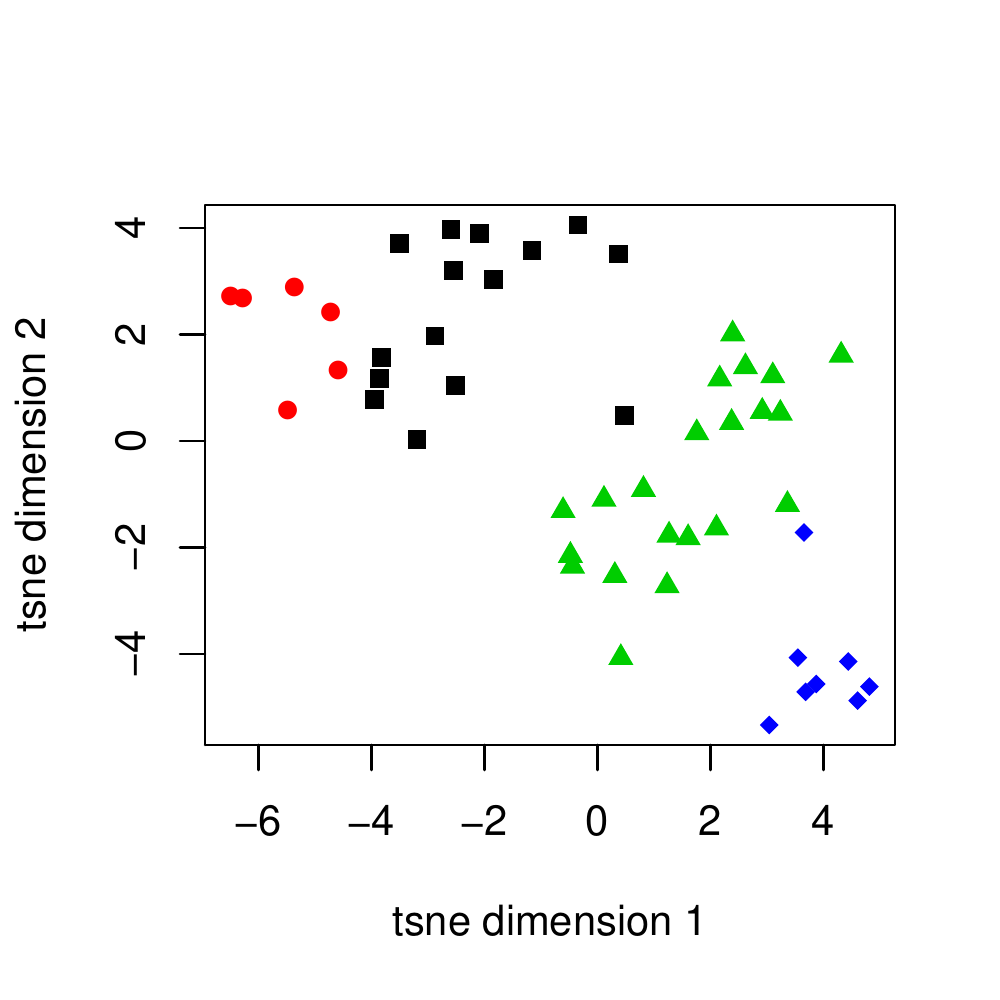}
        \caption{$WK$-means}
        %\label{hist}
    \end{subfigure}
    ~
    \hspace{-17pt} \begin{subfigure}[b]{0.34\textwidth}
    \centering
        \includegraphics[height=\textwidth,width=\textwidth]{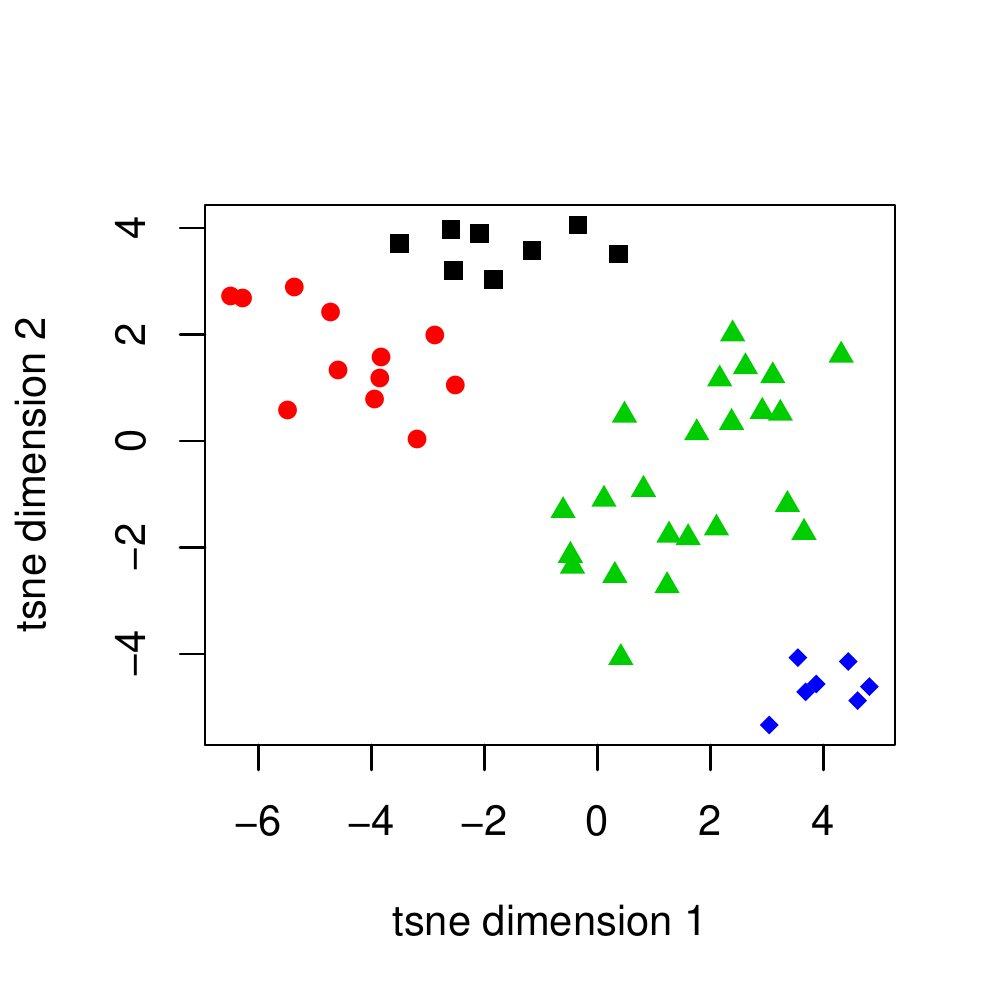}
        \caption{Power $k$-means}
        %\label{gt}
    \end{subfigure}
    ~
    \hspace{-17pt}  \begin{subfigure}[b]{0.34\textwidth}
    \centering
        \includegraphics[height=\textwidth,width=\textwidth]{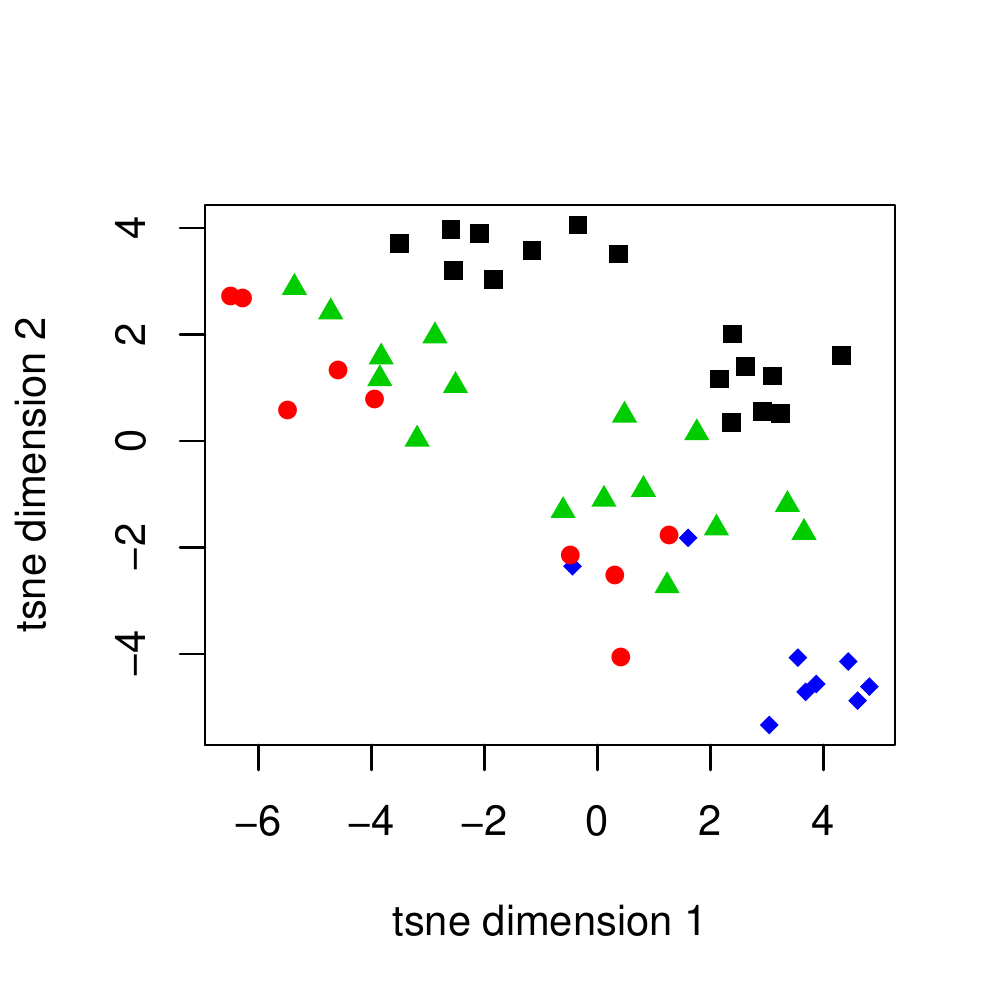}
        \caption{Sparse $k$-means}
        %\label{hist}
    \end{subfigure}
    ~
    \hspace{-17pt}  \begin{subfigure}[b]{0.34\textwidth}
    \centering
        \includegraphics[height=\textwidth,width=\textwidth]{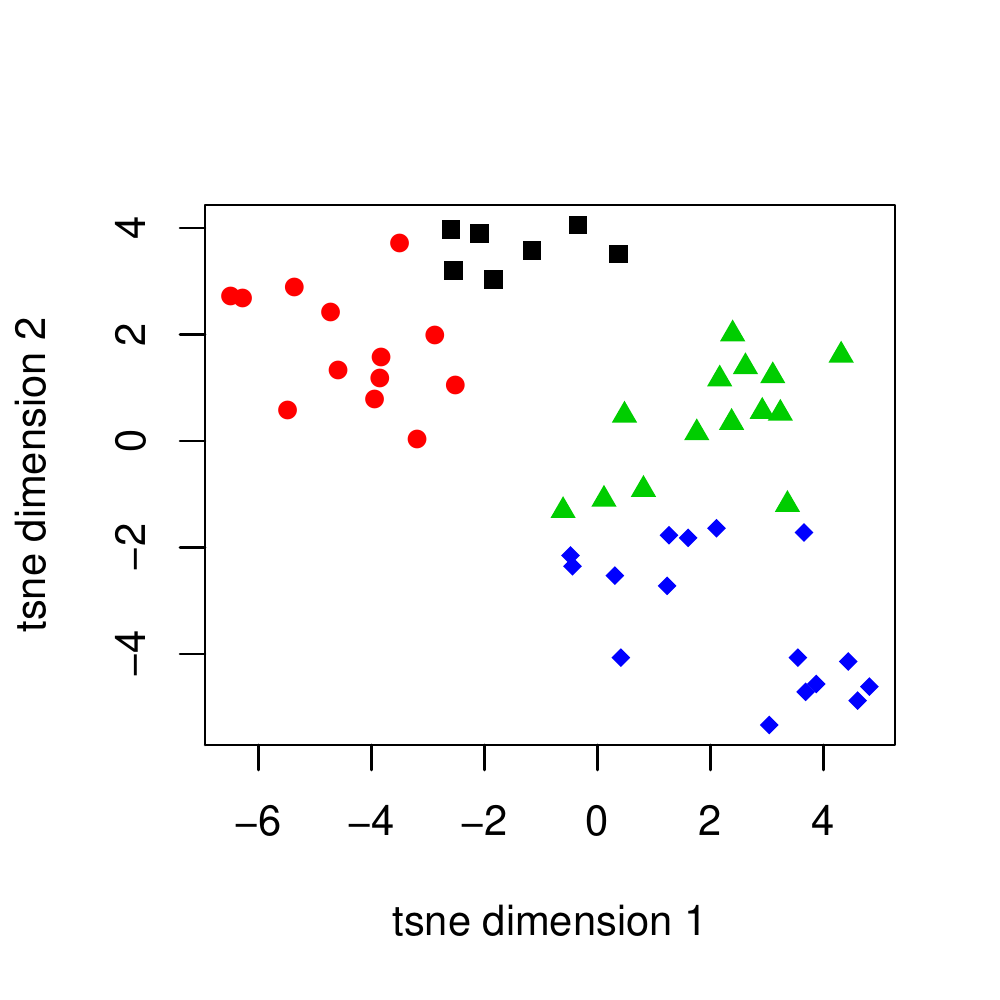}
        \caption{EWP}
        %\label{gt}
    \end{subfigure}
        \caption{\texttt{t-SNE} plots for the GLIOMA dataset, color-coded by the partitioning obtained at convergence by each peer algorithm.}
        \label{tsne_glioma}
        %\vspace{-0.5cm}
\end{figure}

\subsection{Case Study and Real Data}%\vspace{-0.1cm}
\label{glioma}
We now assess performance on real data, beginning with a case study on Glioma. The GLIOMA dataset consists of 50 datapoints and is divided into 4 classes consisting of cancer glioblastomas (CG), noncancer glioblastomas (NG), cancer oligodendrogliomas (CO) and non-cancer oligodendrogliomas (NO). Each observation consists of 4434 features. The data were collected in the study by \cite{nutt2003gene}, and are also available by \cite{li2018feature}.% at \url{http://featureselection.asu.edu/datasets.php}. \par %\href{http://featureselection.asu.edu/datasets.php}{
%Arizona State University Repository (ASU) \citep{li2018feature}.\par
%is a comprehensive dataset in which information on glioma cancer patients are enlisted as instances. Here, glioma is defined as cancer of the brain, cranial nerves or pother nervous system. GLIOMA is a feature selection dataset publicly available at the . GLIOMA is a continuous multi-class dataset with $4$ clusters. This dataset consists of $4434$ many features, while it has only $50$ instances.
In our experimental studies, we compare the EWP algorithm to the four peer algorithms considered in Section \ref{simulation}. In order to visualize  clustering solutions, we embed the data into the plane via \texttt{t-SNE}  \citep{maaten2008visualizing}. The best partitioning obtained from each algorithm is shown in Figure \ref{tsne_glioma}, which makes it visually clear that clustering under EWP more closely resembles the ground truth compared to competitors.
This is detailed by average NMI values as well as standard deviations in parentheses listed in Table \ref{tab_glioma}. 

\begin{table}[h]
    \caption{Mean NMI and standard deviation, GLIOMA}
    \label{tab_glioma}
    \centering
    %\resizebox{0.49\textwidth}{!}{
    \begin{tabular}{|C{2.3cm}|C{2.3cm}|C{2.3cm}|C{2.3cm}|C{2.3cm}|}
       \hline
       $k$-means  & $WK$-means &Power  & Sparse & EWP\\
       \hline
       0.490 (0.040) & 0.427 (0.034) & 0.499 (0.020) & 0.108 (0.001) & \textbf{0.594} (0.001)\\
       \hline
    \end{tabular}
   % }
\end{table}
\paragraph{Further real-data experiments}
To further validate our method in various real data scenarios, we perform a series of experiments on 10 benchmark datasets collected from the UCI machine learning repository \citep{Dua:2019}, Keel Repository \citep{alcala2011keel} and ASU repository \citep{li2018feature}. A brief description of these datasets can be found in Table \ref{data_description}.\par
Average performances over $20$ independent trials are reported in Table \ref{tab_real}; the EWP algorithm outperforms by a large margin across all instances when compared to the other peer algorithms. To determine the statistical significance of the results, we employ  Wilcoxon's signed-rank test \citep{wasserman2006all} at the 5\% level of significance. In Table \ref{tab_real}, an entry marked with $+$ ($\simeq$) differs from the corresponding result of EWP with statistical significance. Finally, we emphasize that our results comprise a conservative comparison in that parameters $s_0=-1$ and $\eta=1.05$ are fixed across \textit{all} settings. While this demonstrates that careful tuning of these parameters is not necessary for successful clustering, performance can be further improved by doing so \citep{xu2019power},

%The statistically significant and insignificant results are marked as $+$ and $\simeq$ in Table \ref{tab_real} respectively. 
%To validate the outcomes of our algorithm more effectively, we compare the Normalized Mutual Index (NMI) values of various real-life datasets for $k$-means, power $k$-means, sparse $k$-means with that of our proposed algorithm (EWP). We run each algorithm in \texttt{R} with the help of various packages as referred earlier in Section \ref{motivation} and \ref{simulation}.
%For choosing the optimal value of hyper-parameters in both power $k$-means and EWP, we used $10$-fold cross validation.
%\par
%The real-life datasets are collected from different sources with varying cluster structure and dimensions. A detailed description of the datasets used are enlisted in Table \ref{data_description}.
%\par
%The average NMI values of each of the algorithms obtained by the above experiment is referred to in Table \ref{tab_real}. The standard deviations are also reported in parenthesis.  
\begin{table}
\centering
\caption{Source and Description of the Datasets}
\label{data_description}
\begin{tabular}{|p{2.8cm}|C{2.8cm}|C{0.5cm}|C{0.7cm}|C{0.7cm}|}
\hline
\textbf{Datasets} & \textbf{Source} & $\mathbf{k}$ & $\mathbf{n}$ & $\mathbf{p}$  \\ 
\hline
Iris & \href{https://sci2s.ugr.es/keel/category.php?cat=clas}{Keel Repository} & 3 & 150 & 4 \\
Automobile & \href{https://sci2s.ugr.es/keel/category.php?cat=clas}{Keel Repository} & 6 & 150 & 25\\
Mammographic & \href{https://sci2s.ugr.es/keel/category.php?cat=clas}{Keel Repository} & 2 & 830 & 5 \\
Newthyroid & \href{https://sci2s.ugr.es/keel/category.php?cat=clas}{Keel Repository} & 3 & 215 & 5\\
Wine & \href{https://sci2s.ugr.es/keel/category.php?cat=clas}{Keel Repository} & 3 & 178 & 13\\
WDBC & \href{https://sci2s.ugr.es/keel/category.php?cat=clas}{Keel Repository} & 2 & 569 & 30\\
Movement Libras & \href{https://sci2s.ugr.es/keel/category.php?cat=clas}{Keel Repository} & 15 & 360 & 90\\
%Cleveland & Keel Repository & 5 & 297 & 13\\ 

%Colon &  &  & & \\
Wall Robot 4 & \href{https://archive.ics.uci.edu/ml/datasets.php}{UCI Repository} & 4 & 5456 & 4\\
%GLIOMA & \href{http://featureselection.asu.edu/datasets.php}{ASU Repository} & 4 & 50 & 4434 \\
WarpAR10P & \href{http://featureselection.asu.edu/datasets.php}{ASU Repository} & 10 & 130 & 2400 \\
WarpPIE10P & \href{http://featureselection.asu.edu/datasets.php}{ASU Repository} & 10 & 210 & 2420\\
%Balance & 0.1204 & 0.0830 & 0.1082 & \textbf{0.1209}\\
%Heart & \href{https://sci2s.ugr.es/keel/category.php?cat=clas}{Keel Repository} & 2 & 270 & 13\\
%Hepatitis & Keel Repository & 2 & 80 & 19\\
%Tae & Keel Repository & 3 & 151 & 5\\
%Mice Protein & UCI Repository & 8 & 1080 & 77\\
%Gastrointestinal Lesions(GL) & UCI Repository & 3 & 152 & 698\\
%GLIOMA & \href{http://featureselection.asu.edu/datasets.php}{ASU Repository} & 4 & 50 & 4434\\
%$COIL20_{10}$ & \href{http://featureselection.asu.edu/datasets.php}{ASU Repository} & 10 & 720 & 1024\\
\hline
%Saheart & 0.0028 & 0.0535 & 0.0310 & \textbf{0.0549}\\
%Phoneme & \textbf{0.1317} & 0.1007 & 0.0806 & 0.1223 \\
%Lymphoma & ASU Repository &   &  & \\
\end{tabular}

\end{table}

% \begin{table}[h]
%     \caption{NMI values on Benchmark Real Data}
%     \label{tab_real}
%     \centering
%     %\resizebox{0.49\textwidth}{!}{
%     \begin{tabular}{|C{2.5cm}|C{1.5cm}|C{1.5cm}|C{1.5cm}|C{1.5cm}|C{1.5cm}|}
%       \hline
%       Dataset & $k$-means& $WK$  &  Power & Sparse & EWP\\
%       \hline
%       Newthyroid & $0.403^+$&  $0.262^+$ & $0.407^+$ & $0.102^+$ & \textbf{0.532}\\
%       \hline
%       Automobile & $0.165^+$& $0.203^+$ & $0.168^+$ & $0.168^+$ & \textbf{0.311}\\
%       \hline
%       WarpAR10P & $0.171+$& $0.233^+$ & $0.201^+$ & $0.185^+$ & \textbf{0.350}\\
%       \hline
%       WarpPIE10P & $0.240^\simeq$ & $0.241^\simeq$ & $0.180^+$ & $0.179^+$ & \textbf{0.276}\\
%       \hline
%       Iris & $0.758^+$ & $0.788^+$ & $0.742^+$ & $0.814^\simeq$ & \textbf{0.849}\\
%       \hline
%       Wine & $0.428^+$ & $0.642^+$ & $0.416^+$ & $0.428^+$ & \textbf{0.747} \\
%       \hline
%       Mammographic & $0.107^+$ & $0.0194^+$ & $0.115^+$ & $0.110^+$ & \textbf{0.405}\\
%       \hline
%       WDBC & $0.463^+$ & $0.005^+$ & $0.464^+$ & $0.467^+$ & \textbf{0.656}\\
%       \hline
%       LIBRAS & $0.553^\simeq$ & $0.339^+$ & $0.461^+$ & $0.254^+$ & \textbf{0.575}\\
%       \hline
%       Wall Robot 4 & $0.168^+$ & $0.184^+$ & $0.171^+$ & $0.186^+$ & \textbf{0.234}\\
%       \hline
      
%     \end{tabular}
% %    }
% \end{table}
\begin{table}
    \caption{NMI values on Benchmark Real Data}
    \label{tab_real}
    \centering
    %\resizebox{0.49\textwidth}{!}{
    \begin{tabular}{|C{3cm}|C{2.5cm}|C{2.5cm}|C{2.5cm}|C{2.5cm}|C{2.5cm}|}
       \hline
       Datasets & $k$-means  &  Power $k$-means & $WK$-means & Sparse $k$-means & EWP-$k$-means\\
       \hline
       Newthyroid & $0.403^+(0.002)$ & $0.262^+(0.002)$ & $0.407^+(0.004)$ & $0.102^+(0.002)$ & \textbf{0.5321}(0.003)\\
       \hline
       Automobile & $0.165^+(0.009)$ & $0.203^+(0.010)$ &  $0.168^+(0.005)$ & $0.168^+(0.007)$ & \textbf{0.311}(0.003)\\
       \hline
       WarpAR10P & $0.170^+(0.042)$ & $0.233^+(0.031)$ & $0.201^+(0.019)$ & $0.185^+(0.008)$ & \textbf{0.350}(0.047)\\
       \hline
       WarpPIE10P & $0.240^\simeq(0.031)$ & $0.240^\simeq(0.028)$ & $0.180^+(0.022)$ & $0.179^+(0.002)$ & \textbf{0.2761}(0.041)\\
       \hline
       Iris & $0.758^+(0.003)$ & $0.788^+(0.005)$ & $0.741^+(0.005)$ & $0.813^\simeq(0.002)$ & \textbf{0.849}(0.005)\\
       \hline
       Wine & $0.428^+(0.001)$ & $0.642^+(0.005)$ & $0.416^+(0.002)$ & $0.428^+(0.001)$ & \textbf{0.747}(0.003) \\
       \hline
       Mammographic & $0.107^+(0.001)$ & $0.019^+(0.003)$ & $0.115^+(0.001)$ & $0.110^+(0.002)$ & \textbf{0.405}(0.002)\\
       \hline
       WDBC & $0.463^+(0.002)$ & $0.005^+(0.005)$ & $0.464^+(0.002)$ & $0.467^+(0.003)$ & \textbf{0.656}(0.001)\\
       \hline
       LIBRAS & $0.553^\simeq(0.017)$ & $0.339^+(0.020)$ & $0.461^+(0.021)$ & $0.254^+(0.014)$ & \textbf{0.575}(0.009)\\
       \hline
       Wall Robot 4 & $0.167^+(0.027)$ & $0.183^+(0.013)$ & $0.171^+(0.030)$ & $0.186^+(0.012)$ & \textbf{0.234}(0.003)\\
       \hline
      
    \end{tabular}
    %}
\end{table}
%\begin{table}[h]
%    \caption{NMI values on Benchmark Real Data}
%    \label{tab_real}
%    \centering
%    \resizebox{0.49\textwidth}{!}{
%%    \begin{tabular}{|C{2.5cm}|C{1.5cm}|C{1.5cm}|C{1.5cm}|C{1.5cm}|C{1.5cm}|}
%       \hline
 %      Datasets & $k$-means& $WK$-means  &  Power $k$-means & Sparse $k$-means & EWP\\
 %      \hline
 %      Newthyroid & $0.4031^+$&  $0.2625^+$ & $0.4072^+$ & $0.1022^+$ & \textbf{0.5321}\\
 %      \hline
 %      Automobile & $0.1655^+$& $0.2034^+$ & $0.1687^+$ & $0.1684^+$ & \textbf{0.3111}\\
 %      \hline
 %      WarpAR10P & $0.1708+$& $0.2334^+$ & $0.2016^+$ & $0.1853^+$ & \textbf{0.3502}\\
 %      \hline
 %      WarpPIE10P & $0.2406^\simeq$ & $0.2407^\simeq$ & $0.1804^+$ & $0.1799^+$ & %\textbf{0.2761}\\
 %      \hline
 %      Iris & $0.7581^+$ & $0.7885^+$ & $0.7419^+$ & $0.8138^\simeq$ & \textbf{0.8498}\\
 %      \hline
 %      Wine & $0.4287^+$ & $0.6427^+$ & $0.4167^+$ & $0.4287^+$ & \textbf{0.7476} \\
 %      \hline
 %      Mammographic & $0.1074^+$ & $0.0194^+$ & $0.1158^+$ & $0.1102^+$ & \textbf{0.4051}\\
 %      \hline
 %      WDBC & $0.4636^+$ & $0.0056^+$ & $0.4648^+$ & $0.4674^+$ & \textbf{0.6564}\\
 %      \hline
 %      LIBRAS & $0.5532^\simeq$ & $0.3390^+$ & $0.4615^+$ & $0.2543^+$ & \textbf{0.5751}\\
 %      \hline
 %      Wall Robot 4 & $0.1677^+$ & $0.1836^+$ & $0.1716^+$ & $0.1861^+$ & \textbf{0.2344}\\
 %%      \hline
  %    
  %  \end{tabular}
  %  }
%\end{table}
%\vspace{-0.3cm}
\section{Discussion}%\vspace{-0.1cm}
\label{discussion}
Despite decades of advancement on $k$-means clustering, Lloyd's algorithm remains the most popular choice in spite of its well-known drawbacks. Extensions and variants that address these flaws fail to preserve its simplicity, scalability, and ease of use. Many of these methods still fall short at poor local optima or fail when data are high-dimensional with low signal-to-noise ratio, and few come with rigorous statistical guarantees such as consistency. 

The contributions in this paper seek to fill this methodological gap, with a novel formulation that draws from good intuition in classic and recent developments. With emphasis on simplicity as a chief priority, we derive a method that can be seen as a drop-in replacement to Lloyd's classic $k$-means algorithm, reaping large improvements in practice even when there are a large number of clusters or features in the data. By designing the algorithm from the perspective of MM, our method is robust as a descent algorithm and achieves an ideal $\mathcal{O}(nkp)$ complexity. In contrast to popular approaches such as sparse $k$-means and power $k$-means, the proposed approach is provably consistent. %Via thorough empirical analysis on a variety of synthetic and real datasets, we demonstrate the merits of our algorithm over classical and state-of-the-art prior work. 

Extending the intuition to robust measures and other divergences in place of the Euclidean distance are warranted.
Further, research toward finite-sample prediction error bounds or convergence rates relating to the annealing schedule will also be fruitful avenues for future work. % incorporate other divergence based distance measures instead of the Euclidean distance to achieve additional non-linearity.
\bibliographystyle{apalike}
%\bibliography{mybib}

\end{document}